\documentclass{article} 
\usepackage{iclr2027_conference,times}

\usepackage{amsmath,amsfonts,bm}

\def\eqref#1{equation~\ref{#1}}

\def\1{\bm{1}}

\DeclareMathAlphabet{\mathsfit}{\encodingdefault}{\sfdefault}{m}{sl}
\SetMathAlphabet{\mathsfit}{bold}{\encodingdefault}{\sfdefault}{bx}{n}

\usepackage{hyperref}
\usepackage{url}

\usepackage{adjustbox,bbm,multirow,multicol}
\usepackage{natbib}
\usepackage{enumitem} 
\usepackage[normalem]{ulem}
\usepackage{algorithm}
\usepackage{bbding} 
\usepackage{makecell}
\PassOptionsToPackage{table}{xcolor} 
\usepackage{colortbl}  
\usepackage{nicematrix, tikz}

\usepackage{booktabs}
\usepackage{multirow}
\usepackage[table]{xcolor}
\usepackage{array}
\usepackage{graphicx}

\usepackage[T1]{fontenc}
\usepackage{lmodern}
\usepackage{amsmath,amssymb,amsthm}
\usepackage{tabularx}
\usepackage{algpseudocode}
\usepackage{pifont}
\usepackage{titletoc}
\usepackage{xcolor}
\usepackage{caption}
\usepackage{wrapfig}
\makeatletter
\newcommand{\authorthanks}[1]{%
  {\def\@makefnmark{\hbox{\@textsuperscript{\normalfont\@thefnmark}}}%
   \thanks{#1}}%
}
\makeatother

\newtheorem{mhgproposition}{Proposition}

\title{Remembering Through Change: Provenance-Preserving Hypergraph Memory for Multimodal Agents}

\author{%
\begin{minipage}[t]{\dimexpr\textwidth-2\tabcolsep\relax}
\raggedright\normalfont
\textbf{%
Jielong Tang$^{1,2}$\authorthanks{Work done during internship.},
Ling-An Zeng$^{2}$\authorthanks{Corresponding authors.},
Xujie Yuan$^{1}$,
Xiao Dong$^{3}$,
Zhiqian Xia$^{4,2*}$,
Qingfeng Li$^{1}$,
Shixiong Kai$^{2}$,
Kaichao Liang$^{2}$,
Rong Yao$^{2}$ and
Jian Yin$^{1\dagger}$}\\[0.7em]
$^{1}$ Sun Yat-sen University \quad
$^{2}$ Huawei Technologies, Noah Arks’ Lab\\
$^{3}$ Shandong Normal University \quad
$^{4}$ Southeast University
\end{minipage}%
}

\iclrfinalcopy 
\usepackage[most]{tcolorbox}
\definecolor{promptink}{HTML}{29343B}
\definecolor{promptsage}{HTML}{E4EDE7}
\definecolor{promptblue}{HTML}{E5ECF3}
\definecolor{promptlavender}{HTML}{EBE8F1}
\definecolor{promptborder}{HTML}{CBD4D5}
\newtcolorbox{mhgprompt}[2]{enhanced,breakable,
 colback=black!1!white,colframe=promptborder,colbacktitle=#1,
 coltitle=promptink,coltext=promptink,boxrule=0.5pt,arc=2mm,
 left=9pt,right=9pt,top=8pt,bottom=8pt,
 fonttitle=\sffamily\bfseries\small,title={#2},
 fontupper=\sffamily\fontsize{9}{10.5}\selectfont,
 before upper={\raggedright\setlength{\parindent}{0pt}\setlength{\parskip}{3pt}},
 before skip=9pt,after skip=8pt}

\begin{document}

\maketitle
\lhead{Preprint.}

\begin{abstract}

Long-term memory for multimodal agents remains challenging, as it must preserve fine-grained visual evidence while tracking entity evolution across sessions. Existing abstraction-based memories compress images into text, often discarding visually discriminative details, whereas pixel-preserving memories retain raw observations but struggle to reliably associate diverse observations with persistent entity identities and track their evolving states. To address these limitations, we propose \textbf{MHG-Mem}, a multimodal hypergraph memory that jointly organizes entity identity, state evolution, and provenance to visual evidence. Specifically, MHG-Mem represents distinct entity mentions and their observed facts as individual nodes, grouping identity-consistent mentions and facts through occurrence hyperedges. It then synthesizes aspect-specific evolution states from these grouped observations and links them to supporting entities, facts, and source evidence through evolutionary and evidential hyperedges. To navigate this structured memory, we introduce hypergraph-based agentic search, which adaptively selects an entry layer, traverses typed hyperedges across memory levels, and inspects source observations when fine-grained verification is required. These designs support both detailed visual verification and reasoning over cross-session state changes. Extensive experiments across three benchmarks demonstrate the effectiveness of MHG-Mem, improving performance by 8.20--9.98\% over the strongest baseline.

\end{abstract}

\section{Introduction}
\label{sec:introduction}

Long-term memory enables large language model (LLM) agents to reuse past experience beyond their context windows~\citep{packer2023memgpt,xu2025amem}. As agents increasingly browse websites and interact with users through images, multimodal memory becomes essential for understanding user needs and providing personalized assistance~\citep{guo2026memeye,liu2026worldmemarena}. As shown in Fig.~\ref{intro}(a), a web agent encounters several backpacks across sessions, with varying page layouts, product views, prices, and availability. Subsequent queries require \textit{inspecting the main zipper of the purchased backpack} (Q1) or \textit{determining whether an earlier offer for the same backpack was both cheaper and in stock} (Q2). Answering these queries remains challenging. Memory must preserve fine-grained visual evidence while tracking the entity's state across sessions despite variations in its visual observation.

\begin{figure}[]
\centering
\includegraphics[width=0.9\linewidth]{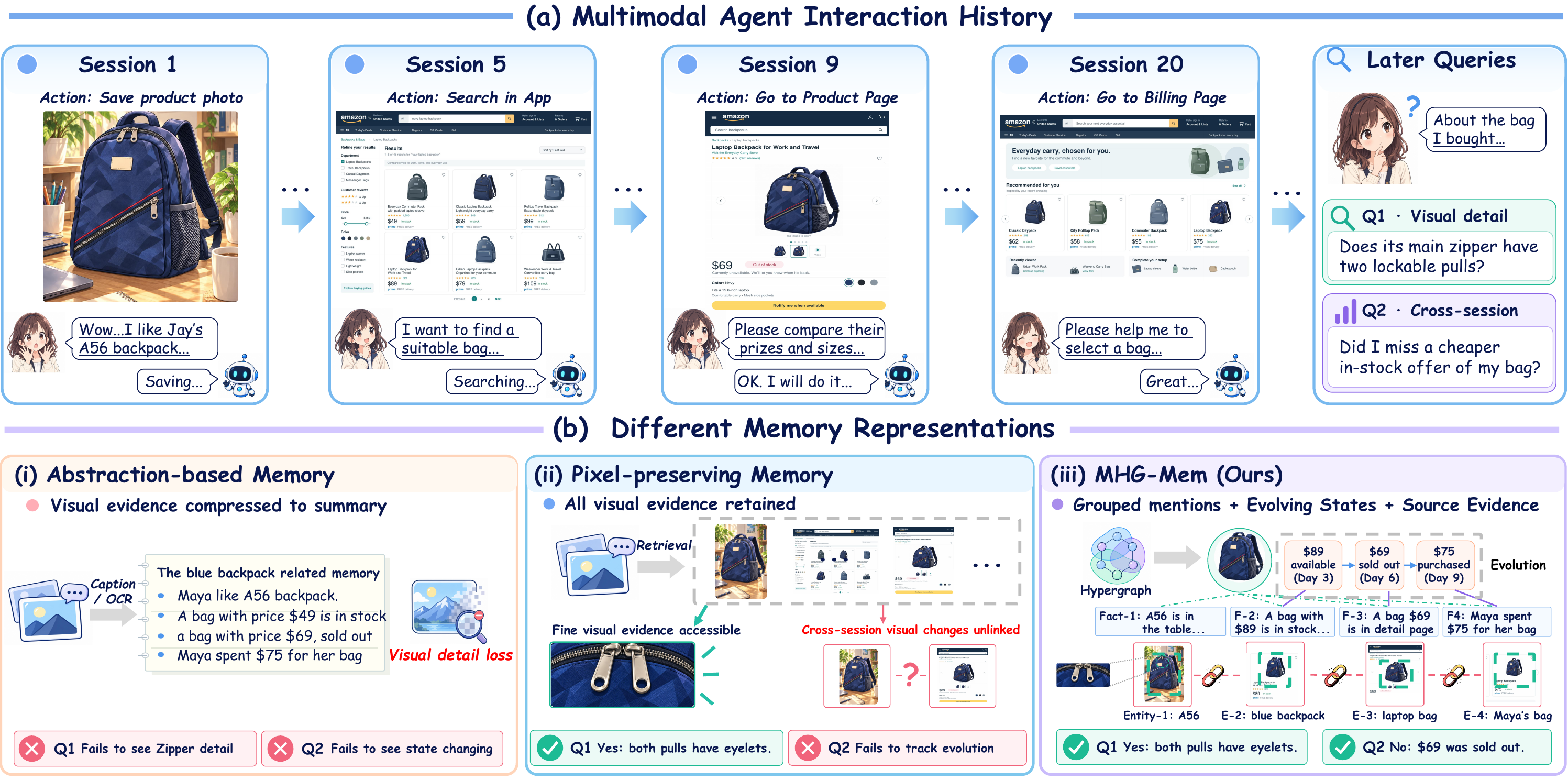}
\vspace{-2pt}
\caption{{(a) Long-term agent--user multimodal interaction example. (b) Comparison of existing approaches: (i) Abstraction-based memory loses visual details (failing Q1, Q2); (ii) Pixel-preserving memory lacks persistent identity--evolution bindings (failing Q2); (iii) Our MHG-Mem organizes multimodal observations via typed hyperedges, preserving both fine-grained visual crops for inspection and traceable evolution across diverse entity mentions.}}
\label{intro}
\vspace{-12pt}
\end{figure}

Existing multimodal memories address these requirements only partially.
As illustrated in Fig.~\ref{intro}(b), \textbf{abstraction-based memory} converts visual observations into textual descriptions or structured records for subsequent consolidation and retrieval~\citep{wang2025mirix,feng2026m2a}.
Such abstractions may omit visual details needed for inspection or for distinguishing similar entities.
\textbf{Pixel-preserving memory} retains original images, allowing relevant observations to be retrieved and inspected at query time~\citep{lu2026mma}.
Yet the retrieved images remain isolated observations. Nothing records that the product photo in one session and the listing page in another show the same backpack, or how its price and availability evolved.
Without reliable associations, an agent may confuse similar entities or incorrectly combine attributes in different observations.
These raise the question:
\textit{How can multimodal memory organize visually diverse observations of the same entity into coherent state histories while keeping each state traceable to its supporting evidence?}

Addressing this question requires preserving individual observations while explicitly organizing how they relate across sessions. The underlying memory structure is inherently \textit{higher-order}: the same entity may be observed differently across sessions, with each observation jointly associating several entity mentions, facts, and visual evidence, while its evolutions must integrate these observations into coherent state histories. These relationships therefore extend beyond isolated entities or facts and naturally form structured groups at different semantic levels. Motivated by this, we propose \textbf{MHG-Mem}, a multimodal hypergraph memory that uses typed hyperedges to explicitly organize these higher-order relationships while preserving each observation separately. Specifically, occurrence hyperedges group identity-consistent entity mentions and their associated facts across sessions, evolution nodes summarize aspect-specific state changes from temporally ordered observations, and evolutionary and evidential hyperedges connect each evolution to the participating entity mentions and its supporting facts, respectively. Because these facts retain provenance to the original dialogue and images, each inferred state remains traceable to its source evidence.

This structured memory naturally exposes evidence at multiple granularities, but different queries require different retrieval paths. Q1 ultimately requires inspecting a particular source image, whereas Q2 requires first locating the backpack's state history and then verifying the facts underlying the relevant price and availability states. A fixed retrieval strategy therefore cannot serve both effectively. To this end, we introduce \textbf{Hypergraph-based Agentic Search}, which adaptively selects an entry layer, traverses typed hyperedges across entity, evolution, and fact-level memory, and follows provenance links to inspect source observations when necessary. In this way, compact evolution records guide cross-session reasoning, while grounded facts and images provide fine-grained, verifiable evidence.
Our contributions are summarized as follows:

\begin{itemize}[leftmargin=1em, labelsep=0.3em, nosep]
    \item {We propose \textbf{MHG-Mem}, a multimodal hypergraph memory framework that organizes cross-session observations into coherent entity evolution histories through typed hyperedges, while preserving visual evidence and state-specific provenance.}
    
    \item {We introduce \textbf{Hypergraph-based Agentic Search}, which adaptively accesses entity groups, state histories, and their supporting observations to assemble query-relevant evidence across granularities.}
    
    \item {Experiments across three benchmarks and three backbones demonstrate MHG-Mem's effectiveness, improving average performance by 8.20–9.98\% over the strongest baseline.}
\end{itemize}

\section{Related Work}
\label{sec:related_work}

\textbf{Long-Term Multimodal Agent Memory.}
As agents increasingly interact with the physical and digital world, long-term memory research has expanded beyond text-centric systems~\citep{packer2023memgpt,xu2025amem, kang2025memoryos,chhikara2025mem0, zhong2024memorybank, yuan2026metakgrag} to multimodal retrieval~\citep{chen2022murag,yeo2025universalrag, tang2026sake}, episodic video memory~\citep{jeong2025videorag,fan2024videoagent}, and multimodal conversational memory~\citep{bei2026memgallery,maharana2024evaluating}. Recent benchmarks further examine whether agents can retain action trajectories, and visually evolving states across long-horizon interactions~\citep{liu2026worldmemarena,guo2026memeye}. Existing systems address important aspects of this problem by compressing observations into semantic abstractions~\citep{wang2025mirix,feng2026m2a,jain2025augustus}, organizing experiences around entities or events for cross-session reasoning~\citep{long2025m3agent,chen2026em2mem,nguyen2026visualmem,liu2026omnisimplemem}, and retrieving evidence across modalities or memory stores~\citep{kang2026vmem, lu2026mma,tian2026adamm}.
Yet these capabilities remain difficult to reconcile: semantic abstraction improves memory organization but may discard visual cues needed to distinguish similar observations, while preserving raw pixels alone does not establish which observations refer to the same entity or how their attributes evolve over time. 
MHG-Mem jointly addresses these challenges by preserving multimodal occurrences, associating them with persistent entity identities, synthesizing aspect-specific state histories, and grounding each state in supporting facts and source images within a typed hypergraph, enabling agentic search to traverse between compact abstractions and verifiable evidence.

\textbf{Hypergraph-based Agent Memory.} Graph memory has evolved from knowledge graphs~\citep{rasmussen2025zep, gutierrez2024hipporag}, toward entity-centered graph such as M3-Agent~\citep{long2025m3agent} and query-aware evidence expansion in GraphMemix~\citep{li2026graphmemix}. Hypergraph methods then make higher-order membership explicit: HyperGraphRAG~\citep{luo2025hypergraphrag} represents $n$-ary facts for corpus-level RAG, HGMEM~\citep{zhou2025hgmem} uses hypergraph working memory for multi-step reasoning, and HyperMem~\citep{yue2026hypermem} organizes long conversations as topic--episode--fact hypergraphs with coarse-to-fine retrieval. Despite their effectiveness, these approaches remain largely confined to text-centric memory organization, limiting their applicability to information-dense and dynamically evolving multimodal agent interactions. In contrast, MHG-Mem uses typed hyperedges to jointly organize persistent multimodal occurrences, entity-specific aspect evolutions and provenance, enabling identity-aware state tracking and fine-grained visual verification within a unified long-term memory framework.



\section{Methodology}
\label{sec:method}

To track entity states across sessions while preserving their supporting evidence, we propose MHG-Mem, a multimodal hypergraph memory framework. Typed hyperedges organize entity mentions, facts, and state evolutions into reusable histories while retaining observation-specific visual evidence and links from individual states to their sources. Hypergraph-based agentic search then adapts retrieval to the query, navigating these histories and inspecting the underlying evidence as needed. Figure \ref{fig:method} illustrates (a) memory construction, (b) the multimodal hypergraph memory structure, and (c) hypergraph-based agentic search.

\begin{figure}[t]
\centering
\includegraphics[width=1.0\linewidth]{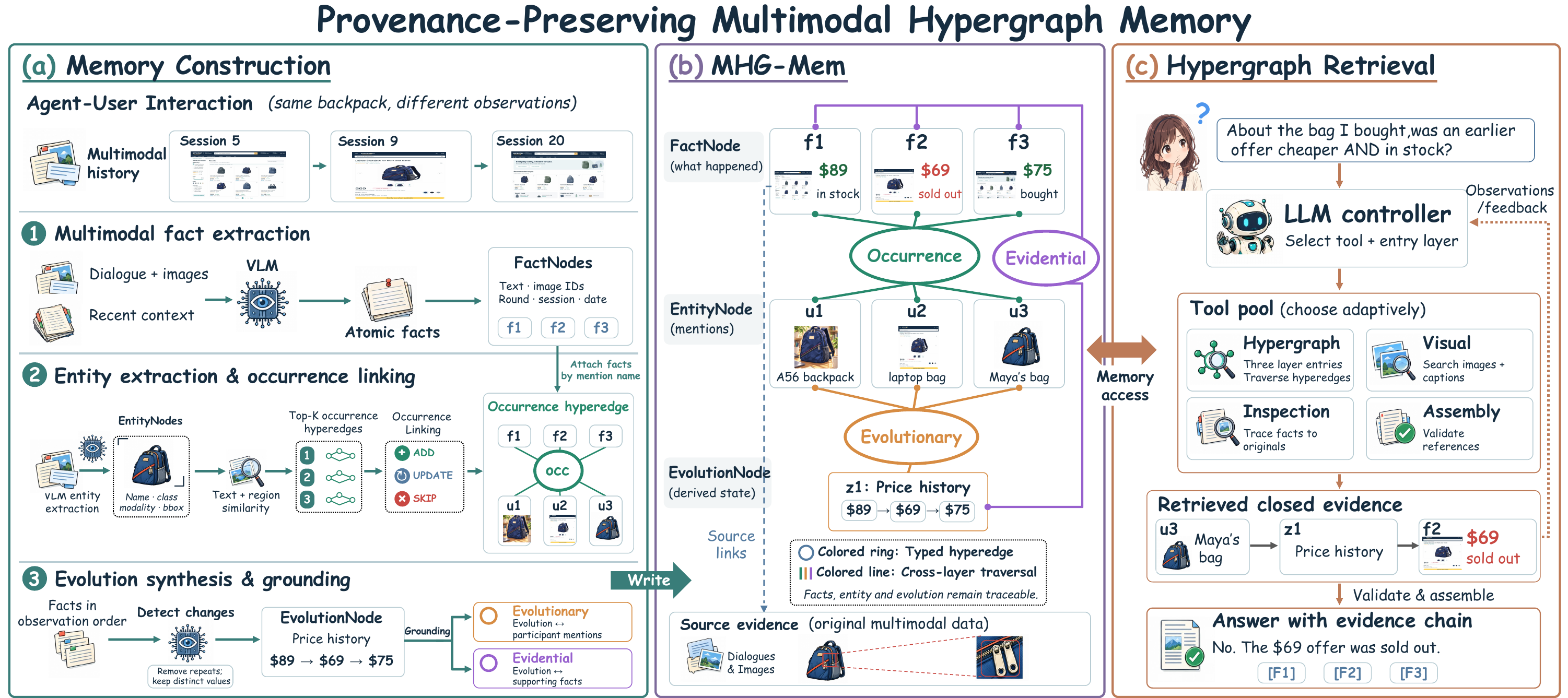}
\caption{Overview of \textbf{MHG-Mem}. (a) Memory construction proceeds through three steps. (b) The three-layer MHG-Mem architecture and three types of hyperedges enable cross-layer traversal and provenance tracing. (c) Hypergraph-based agentic search employs four types of tools to adaptively construct query-aware evidence chains.}
\label{fig:method}
\vspace{-10pt}
\end{figure}

\subsection{Problem Formulation}
\label{sec:method-problem}

Let $\mathcal{D}_{\leq T}=(d_1,\ldots,d_T)$ denote an agent’s ordered multimodal interaction history across sessions
, where each round $d_t=(x_t,\mathcal{I}_t,\delta_t)$ comprises dialogue text $x_t$, images $\mathcal{I}_t$, and a recorded date $\delta_t$. A memory construction operator $\mathcal{W}$ maps this history to persistent memory $\mathcal{M}_T=\mathcal{W}(\mathcal{D}_{\leq T})$ before subsequent queries are known. Given a query $q=(x_q,\mathcal{I}_q)$ 
with optional visual input $\mathcal{I}_q$, 
a retrieval operator $\mathcal{S}$ selects relevant evidence from $\mathcal{M}_T$. A multimodal agent then generates an answer conditioned on the query and retrieved evidence:
\begin{equation}
    {E}_q=\mathcal{S}(q,\mathcal{M}_T),
    \qquad
    \hat{y}\sim p_\theta(\,\cdot\mid q,{E}_q),
    \label{eq:mhg-pipeline}
\end{equation}

where ${E}_q$ denotes the retrieved evidence and $\theta$ denotes fixed model parameters. The objective is to design $\mathcal{W}$ and $\mathcal{S}$ to support accurate answering of multimodal agent.

\subsection{Multimodal Hypergraph Memory Structure}
MHG-Mem represents multimodal memory using three node types and three types of $n$-ary hyperedges. Nodes encode memory units at different levels of granularity, while hyperedges capture cross-level associations and provenance, allowing high-level memory to be traced back to their supporting evidence. 
We formalize this memory as an attributed hypergraph: 
\begin{equation}
\resizebox{0.85\linewidth}{!}{$\displaystyle
\mathcal{H}=(\mathcal{V},\mathcal{E},\iota,\tau_V,\tau_E),\quad
\mathcal{V}=\mathcal{F}\,\dot\cup\,\mathcal{U}\,\dot\cup\,\mathcal{Z},\quad
\mathcal{E}=\mathcal{E}_{\mathrm{occ}}\,\dot\cup\,\mathcal{E}_{\mathrm{evo}}\,\dot\cup\,\mathcal{E}_{\mathrm{evd}},\quad
\iota:\mathcal{E}\rightarrow 2^{\mathcal{V}}$}
\label{eq:mhg-hypergraph}
\end{equation}
where $\iota(e)$ denotes the set of nodes incident to hyperedge $e$, while
$\tau_V$ and $\tau_E$ specify the node and hyperedge types, respectively. The contribution of nodes and hyperedges is defined as follows: \ding{182} \texttt{FactNodes} store atomic multimodal observations.
\texttt{\ding{183} EntityNodes} store individual textual or visual entity mentions. We denote EntityNodes linked to the same underlying entity as a \texttt{referent group} $r$. \texttt{\ding{184} EvolutionNodes} store state changes in an entity attribute or relation. \texttt{\ding{185} Occurrence hyperedges} bind a referent's EntityNodes to its associated FactNodes. \texttt{\ding{186} Evolutionary hyperedges} connect EvolutionNodes to its associated EntityNodes. \texttt{\ding{187} Evidential hyperedges} connect EvolutionNodes to the supporting FactNodes used for verification.

The complete node, hyperedge, and incidence definitions are given in Appendix~\ref{app:mhg-schema}. The notation is summarized in Appendix~\ref{app:notation}. This memory structure supports cross-session
identity tracking, state synthesis, and provenance tracing while preserving original visual evidence.

\subsection{MHG-Mem Construction}
\label{sec:method-construction}

\textbf{Step~1: Multimodal Fact Extraction.}
In each round $d_t$, a vision-language model (VLM) extracts atomic multimodal fact descriptions
from dialogue and images, using recent within-session context
$\mathcal{C}_t$ to resolve references. 
These descriptions are grounded
as FactNodes:
\begin{equation}
    \mathcal{F}_t=\operatorname{Ground}_t\!\left(\operatorname{VLM}_{\mathrm{fact}}(d_t;\mathcal{C}_t)\right), 
    \label{eq:mhg-fact-extraction}
\end{equation}
where $\operatorname{Ground}_t$ creates each
$f=(\ell_f,\mathcal{I}_f, t_f,s_f,\delta_f)$ by attaching the current round's
image identifiers and source metadata to the extracted description $\ell_f$.

\textbf{Step~2: Multimodal Entity Extraction and Occurrence Linking.}
A VLM extracts multimodal entity mentions from $d_t$.
For each mention $u$, we score an occurrence hyperedge by independently
maximizing textual and visual similarity over its constituent EntityNodes:
\begin{equation}
    s(u,e_r^{\mathrm{occ}})=w_T\max_{v\in\mathcal{U}_r}c_T(u,v)+w_I\max_{v\in\mathcal{U}_r^{\mathrm{vis}}}c_I(u,v).
    \label{eq:mhg-link-score}
\end{equation}
Here, $c_T$ and $c_I$ are cosine similarities between mention-name and
image-region embeddings, respectively, and $\mathcal{U}_r^{\mathrm{vis}}$
contains the visual mentions in $\mathcal{U}_r$.  Conditioned on $d_t$, $u$, and the top-$K$ hyperedges ranked by $s$, the model
selects one of three actions for occurrence linking: \texttt{UPDATE} adds
$u$ to the mention set $\mathcal{U}_r$ of a matching referent group,
\texttt{ADD} creates a new group $r$ with $\mathcal{U}_r=\{u\}$ and its
occurrence hyperedge, and \texttt{SKIP} rejects an invalid entity mention.
Each accepted mention remains a distinct EntityNode, retaining its visual
region when available. We then attach facts in $\mathcal{F}_t$ to occurrence
hyperedges by mention-name matching:
\begin{equation}
    \mathcal{F}_r\leftarrow\mathcal{F}_r\cup\{f\in\mathcal{F}_t:A(f,e_r^{\mathrm{occ}})=1\}.
    \label{eq:mhg-occurrence-update}
\end{equation}
Here, $A(f,e_r^{\mathrm{occ}})\in\{0,1\}$ is a lexical matching
indicator that equals 1 iff $\ell_f$ contains the name of at least
one EntityNode in $\mathcal{U}_r$. The resulting membership
is $\iota(e_r^{\mathrm{occ}})=\mathcal{U}_r\cup\mathcal{F}_r$.

\textbf{Step~3: Evolution Synthesis and Grounding.}
To track the evolving states of the same entity across sessions, we
synthesize the facts grouped by each occurrence hyperedge into state
histories with explicit evidence grounding. For each occurrence hyperedge $e_r^{\mathrm{occ}}$, candidate EvolutionNodes $\widehat{\mathcal{Z}}_r$ are generated from the associated
facts $\mathcal{F}_r$:
\begin{equation}
    \widehat{\mathcal{Z}}_r=
      \operatorname{LLM}_{\mathrm{evo}}\!\left(
        \operatorname{Order}(\mathcal{F}_r)\right),
    \label{eq:mhg-evolution-synthesis}
\end{equation}
where $\operatorname{Order}$ follows observation-order $t_f$ in
$\mathcal{D}_{\leq T}$. The model identifies changes in the same attribute
or relation $a_z$ across sessions, excluding repeated descriptions and
incidental scene variation. Each retained $z \in\widehat{\mathcal{Z}}_r$ records the subject referent
$r_z=r$, summary $\sigma_z$, update type $\kappa_z$ (replacement, update,
or addition), and state sequence
\begin{equation}
    S_z=\bigl((v_j,f_j)\bigr)_{j=1}^{m_z},\qquad
    \mathcal{F}_z=\{f_j:1\leq j\leq m_z\}\subseteq\mathcal{F}_{r_z},
    \label{eq:mhg-evolution}
\end{equation}
where $v_j$ is the changing value extracted for $a_z$, and the
$m_z$ states follow observation order. Retention requires valid references within $\mathcal{F}_{r_z}$ and at least two distinct values.

The evolutionary hyperedge $e_z^{\mathrm{evo}}$ then connects $z$ to subject
mentions $\mathcal{U}_z^{\mathrm{subj}}\subseteq\mathcal{U}_{r_z}$ and coactor
mentions $\mathcal{U}_z^{\mathrm{co}}$. Coactors come from cited rounds of
other groups $r'\ne r_z$ sharing supporting facts,
$\mathcal{F}_{r'}\cap\mathcal{F}_z\ne\varnothing$.
The evidential hyperedge $e_z^{\mathrm{evd}}$ links $z$ to its supporting
facts $\mathcal{F}_z$, with $S_z$ retaining the individual state--fact
correspondence. Please refer to Appendix~\ref{app:mhg-construction-details} for more details of MHG-Mem construction.

\subsection{Hypergraph-based Agentic Search}
\label{sec:method-search}
Since different queries require integrating evidence at different levels of granularity, fixed retrieval and graph traversal still struggle to get the optimal evidence chain~\citep{luo2025graph}. We therefore implement $\mathcal{S}$ as agentic search that adaptively selects a memory layer, traverses hyperedges across memory layers, and follows provenance links to sources.

\textbf{Adaptive Search.}
Given a multimodal query $q$ with optional images and search history
$h_k=((a_i,o_i))_{i<k}$, an LLM controller $\pi_\theta$ selects
a tool call $a_k$ and receives observation $o_k$:
\begin{equation}
    a_k\sim\pi_\theta(\,\cdot\mid q,h_k),\qquad
    o_k=\operatorname{Execute}(a_k,\mathcal{M}_T).
    \label{eq:mhg-controller}
\end{equation}
Four tool types support this agent loop: (i) Hypergraph Tool (ii) Visual Tool (iii) Inspection Tool and (iv) Assembly Tool. Table~\ref{tab:mhg-tools} shows the tool details, the search process is outlined in Algorithm~\ref{alg:mhg-search}.

\textbf{Hypergraph Retrieval.}
Each memory layer provides an isolated search tool (as described in Table~\ref{tab:mhg-tools}), with hyperedge traversal
among retrieved entities, evolutions and facts.

\emph{Fact-layer Search.} For a search query $w$ generated by the LLM,
we rank $f\in\mathcal{F}$ using
\begin{equation}
    s_F(w,q,f)=\lambda_Ts_T(w,f)+\lambda_{TI}s_{TI}(w,f)
                         +\lambda_{II}s_{II}(q,f),
    \label{eq:mhg-fact-score}
\end{equation}
where $s_T$, $s_{TI}$, and $s_{II}$ measure text--text, text--image,
and image--image cosine similarity with weights $\lambda_T$,
$\lambda_{TI}$, and $\lambda_{II}$. Visual terms take the best match
over associated images, including $\mathcal{I}_q$ for $s_{II}$.
Retrieved facts retain their source references for inspection.

\emph{Entity-layer Search.} A name or query image identifies referent groups
through maximum textual similarity to referent entity names and visual similarity
over $\mathcal{U}_r^{\mathrm{vis}}$ (equation~\ref{eq:mhg-link-score}). 
We then traverse each matched occurrence hyperedge $e_r^{\mathrm{occ}}$
from the entity layer to the fact layer $\mathcal{F}_r=\iota(e_r^{\mathrm{occ}})\cap\mathcal{F}$.
This gathers observations of the same underlying entity across rounds.

\emph{Evolution-layer Search.} To retrieve how an entity changes over time,
we first identify its referent group $r$ from a name or query image and
retrieve the EvolutionNodes belonging to that group ($r_z=r$).
The evolutionary hyperedge $e_z^{\mathrm{evo}}$ connects each such node
to mentions of its subject entity. We then compare the query text $x_q$
with each node's attribute $a_z$ and summary $\sigma_z$ to assess relevance:
\begin{equation}
    s_Z(q,z)=\cos\bigl(\phi_T(x_q),\phi_T(a_z\mathbin{\|}\sigma_z)\bigr).
    \label{eq:mhg-evolution-score}
\end{equation}
Here, $\phi_T$ encodes text, $\cos(\cdot)$ denotes cosine similarity, and $\|$ concatenates the attribute and summary.
For each selected $z$, we traverse its evidential hyperedge
$e_z^{\mathrm{evd}}$ from the evolution layer to the fact layer,
getting the supporting facts $\mathcal{F}_z$ in the order recorded by $S_z$.

\emph{Cross-layer Traversal.}
A traversal step moves from a retrieved node $v\in\mathcal{V}$ to another
memory layer through incident hyperedges. Given selected hyperedges
$\mathcal{E}'\subseteq\mathcal{E}$ and a target layer
$\mathcal{V}'\in\{\mathcal{F},\mathcal{U},\mathcal{Z}\}$ with
$v\notin\mathcal{V}'$, the cross-layer traversal can be formalized as:
\begin{equation}
    \scalebox{0.85}{$\displaystyle
    \operatorname{Traverse}(v,\mathcal{E}',\mathcal{V}')=
    \bigcup\nolimits_{e\in\mathcal{E}',\,v\in\iota(e)}
        (\iota(e)\cap\mathcal{V}').
    $}
    \label{eq:mhg-traverse}
\end{equation}

\textbf{Source Inspection and Evidence Assembly.}
The Inspection Tool follows $\rho(\cdot)$ from retrieved FactNodes to the original observations, enabling the controller to inspect fine-grained evidence for queries that require detailed information or visual verification. For an EvolutionNode $z$, traversing $e_z^{\mathrm{evd}}$ to supporting
facts $f\in\mathcal{F}_z$ and then following $\rho(f)$ keeps its state
history traceable to the original observations. The Assembly Tool organizes
selected evidence and image references, together with optional summary, into
an evidence chain $E_q$. After reference validation, this chain is submitted
to the QA model with the cited textual and visual evidence to answer $q$
as in Equation~\eqref{eq:mhg-pipeline}.

\section{Experiment}

\subsection{Experimental Setup}
\label{sec:experiment-setup}

\textbf{Benchmarks and Metrics.}
We evaluate on three multimodal memory benchmarks:
MemEye~\citep{guo2026memeye}, Mem-Gallery~\citep{bei2026memgallery}, and
WorldMemArena-S (subset)~\citep{liu2026worldmemarena}.
We report exact match (EM) for MemEye's
multiple-choice (MCQ) questions, token-level F1 for the other two benchmarks,
and BLEU-1 (B-1) and LLM-as-a-judge (Judge) score for open-ended answers on all
three. 
For memory quality evaluation, we follow~\citep{liu2026worldmemarena} using the same judging rubrics, recall, and retrieval coverage. Dataset and metric details are provided in Appendix~\ref{app:experiment-benchmarks}.

\textbf{Implementation Details.}
We use Qwen3-VL-8B~\citep{bai2025qwen3vl}, GPT-4.1-mini~\citep{openai2025gpt41mini}, and Qwen3.8-27B~\citep{qwen2026qwen38} as different memory backbones, GPT-5.4-mini~\citep{openai2026gpt54mini} as LLM-as-a-judge model.
MHG-Mem uses all-MiniLM-L6-v2 for text embeddings and SigLIP2~\citep{tschannen2025siglip2}  for visual and cross-modal embeddings.
For a fair comparison, we reproduce all baseline methods with the same model setting and running environment of MHG-Mem. Results are averaged over 3 runs. Additional implementation settings are provided in Appendix~\ref{app:experiment-implementation}.

\textbf{Baselines.}
We compare with four unimodal memory methods using dense image caption:
A-Mem~\citep{xu2025amem}, MemoryOS~\citep{kang2025memoryos},
SimpleMem~\citep{liu2026simplemem}, and HyperMem~\citep{yue2026hypermem};
and four multimodal memory or RAG methods:
M2A~\citep{feng2026m2a}, MMA~\citep{lu2026mma},
SRAG~\citep{guo2026memeye}, and UniversalRAG~\citep{yeo2025universalrag}.
Detailed baseline descriptions are deferred to
Appendix~\ref{app:experiment-baselines}.

\definecolor{BackboneGray}{HTML}{E9ECEF}
\definecolor{OursBG}{HTML}{F3F7FD}
\definecolor{GainBlue}{HTML}{2563EB}

\newcommand{\ph}{--}

\newcommand{\judgegain}[1]{\,{\color{blue}\scriptsize\textbf{(+#1)}}}

\newcommand{\backbonerow}[1]{%
    \rowcolor{BackboneGray}
    \multicolumn{13}{c}{%
        \rule{0pt}{2.35ex}%
        \textit{#1}%
        \rule[-0.75ex]{0pt}{0pt}
    }\\
}

\begin{table*}[t]
\centering

\caption{
\textbf{Comparison with state-of-the-art memory systems across different backbone models.}
We report Exact Match (\textbf{EM}), BLEU-1 (\textbf{B-1}), and LLM-Judge (\textbf{Judge})
on \textbf{MemEye}, \textbf{Mem-Gallery}, and \textbf{WorldMemArena-S}.
\textbf{Avg. J.} denotes the average Judge score across the three benchmarks. Bold and underlined results indicate the best and second-best performance, respectively.
}
\label{tab:main_results}


\setlength{\tabcolsep}{4.0pt}
\renewcommand{\arraystretch}{1.08}

\resizebox{\linewidth}{!}{%
\begin{tabular}{llc*{10}{c}}

\toprule

\multirow{2}{*}{\textbf{Type}}
&
\multirow{2}{*}{\textbf{Method}}
&
\multirow{2}{*}{\textbf{Venue}}
&
\multicolumn{3}{c}{\textbf{MemEye}}
&
\multicolumn{3}{c}{\textbf{Mem-Gallery}}
&
\multicolumn{3}{c}{\textbf{WorldMemArena-S}}
&
\multirow{2}{*}{\textbf{Avg. J.}}
\\

\cmidrule(lr){4-6}
\cmidrule(lr){7-9}
\cmidrule(lr){10-12}

&
&
&
\textbf{EM} &
\textbf{B-1} &
\textbf{Judge}
&
\textbf{F1} &
\textbf{B-1} &
\textbf{Judge}
&
\textbf{F1} &
\textbf{B-1} &
\textbf{Judge}
&
\\

\midrule

\backbonerow{Qwen3-VL-8B}

\multirow{4}{*}{Unimodal}
& A-Mem
& NIPS'25
& 45.15 & 19.22 & 29.38
& 47.68 & 41.76 & \underline{66.93}
& \textbf{25.23} & 12.16 & \underline{46.57}
& 47.63
\\

& MemoryOS
& EMNLP'25
& 40.23 & 17.22 & 26.28
& 39.36 & 35.28 & 58.86
& 10.24 & 3.50 & 20.18
& 35.11
\\

& SimpleMem
& ICML'26
& 44.27 & 16.04 & 31.26
& 37.36 & 33.12 & 52.43
& 14.33 & 4.71 & 33.50
& 39.06
\\

& HyperMem
& ACL'26
& 42.31 & 16.84 & 28.42
& 41.05 & 32.80 & 62.56
& 20.06 & \underline{12.33} & 40.45
& 43.81
\\

\cmidrule(lr){1-13}

\multirow{4}{*}{Multimodal}
& M2A
& Arxiv'26
& 39.96 & 10.40 & 18.41
& 18.19 & 14.03 & 37.65
& 23.02 & 11.31 & 41.64
& 32.57
\\

& MMA
& AACL’26
& 46.92 & 20.62 & 38.17
& 34.80 & 30.34 & 57.49
& 21.73 & 10.82 & 38.95
& 44.87
\\

& SRAG
& Arxiv’26 
& 56.78 & \underline{24.87} & 41.80
& 40.74 & 35.89 & 63.62
& \underline{23.53} & 11.21 & 43.93
& \underline{49.78}
\\

& UniversalRAG
& ACL’26
& \underline{57.48} & 23.92 & \underline{42.09}
& \underline{51.05} & \underline{46.39} & 66.23
& 14.96 & 5.06 & 36.33
& 48.21
\\

\cmidrule(lr){1-13}

\rowcolor{OursBG}
\textbf{Ours}
& \textbf{MHG-Mem}
& \ph
& \textbf{61.26} & \textbf{26.52} & \textbf{46.28}
& \textbf{66.00} & \textbf{60.90} & \textbf{76.89}
& 23.17 & \textbf{15.29} & \textbf{50.78}
& \textbf{57.98}\judgegain{8.20}
\\

\specialrule{0.9pt}{0.25em}{0.25em}

\backbonerow{GPT-4.1-mini}

\multirow{4}{*}{Unimodal}
& A-Mem
& NIPS’25
& 49.73 & 21.11 & 38.27
& 44.08 & 38.37 & 64.35
& \underline{33.39} & 17.54 & \underline{60.40}
& 54.34
\\

& MemoryOS
& EMNLP’25
& 45.75 & 13.12 & 34.91
& 43.09 & 37.76 & 64.64
& 17.46 & 6.63 & 35.87
& 45.14
\\

& SimpleMem
& ICML’26 
& 50.34 & 9.42 & 42.59
& 34.73 & 28.90 & 58.33
& 23.75 & 14.71 & 56.13
& 52.35
\\

& HyperMem
& ACL’26 
& 47.82 & 19.64 & 36.43
& 44.57 & 33.42 & 66.76
& 24.91 & 15.62 & 52.04
& 51.74
\\

\cmidrule(lr){1-13}

\multirow{4}{*}{Multimodal}
& M2A
& Arxiv’26
& 52.31 & 6.31 & 44.50
& 15.04 & 10.28 & 51.81
& 29.26 & \underline{18.55} & 53.38
& 49.90
\\

& MMA
& AACL’26
& 50.96 & 16.32 & 43.81
& 26.99 & 20.60 & 60.87
& 30.10 & 16.33 & 51.15
& 51.95
\\

& SRAG
& Arxiv’26 
& 58.36 & \textbf{31.48} & 50.54
& 41.07 & 35.03 & 67.06
& 30.47 & 18.34 & 54.27
& 57.29
\\

& UniversalRAG
& ACL’26
& \underline{60.49} & 11.23 & \underline{52.61}
& \underline{48.36} & \underline{41.90} & \underline{70.63}
& 26.44 & 13.03 & 50.95
& \underline{58.06}
\\

\cmidrule(lr){1-13}

\rowcolor{OursBG}
\textbf{Ours}
& \textbf{MHG-Mem}
& \ph
& \textbf{66.84} & \underline{22.36} & \textbf{59.73}
& \textbf{52.58} & \textbf{45.75} & \textbf{81.56}
& \textbf{33.61} & \textbf{21.52} & \textbf{62.57}
& \textbf{67.95}\judgegain{9.89}
\\

\specialrule{0.9pt}{0.25em}{0.25em}

\backbonerow{Qwen3.8-27B}

\multirow{4}{*}{Unimodal}
& A-Mem
& NIPS’25
& 45.10 & 18.13 & 33.64
& \underline{60.56} & \underline{54.37} & 75.98
& 27.99 & 13.92 & 57.49
& 55.71
\\

& MemoryOS
& EMNLP’25
& 42.24 & 16.67 & 28.25
& 42.82 & 37.49 & 70.19
& 10.17 & 3.76 & 20.95
& 39.80
\\

& SimpleMem
& ICML’26
& 47.84 & 5.37 & 37.40
& 49.43 & 43.26 & 68.70
& 18.95 & 7.34 & 46.33
& 50.81
\\

& HyperMem
& ACL’26
& 50.56 & 21.02 & 39.66
& 47.23 & 36.64 & 73.03
& 22.84 & 10.61 & 50.57
& 54.42
\\

\cmidrule(lr){1-13}

\multirow{4}{*}{Multimodal}
& M2A
& Arxiv’26
& 48.23 & 19.51 & 35.90
& 20.46 & 15.49 & 65.30
& 28.15 & \underline{14.14} & 55.58
& 52.25
\\

& MMA
& AACL’26
& 56.47 & 24.95 & 47.30
& 35.18 & 29.03 & 69.32
& 26.66 & 13.35 & 52.80
& 56.47
\\

& SRAG
& Arxiv’26 
& 66.95 & \underline{28.26} & 53.37
& 41.45 & 35.68 & 75.80
& \underline{28.78} & 13.97 & \underline{59.69}
& 62.95
\\

& UniversalRAG
& ACL’26
& \underline{67.85} & 13.45 & \underline{60.64}
& 57.30 & 50.51 & \underline{76.18}
& 22.11 & 9.34 & 53.07
& \underline{63.30}
\\

\cmidrule(lr){1-13}

\rowcolor{OursBG}
\textbf{Ours}
& \textbf{MHG-Mem}
& \ph
& \textbf{78.76} & \textbf{39.54} & \textbf{65.93}
& \textbf{66.07} & \textbf{59.84} & \textbf{87.85}
& \textbf{30.48} & \textbf{21.21} & \textbf{66.05}
& \textbf{73.28}\judgegain{9.98}
\\

\bottomrule

\end{tabular}%
}

\end{table*}

\subsection{Main Results}
\label{sec:main-results}

Table~\ref{tab:main_results} compares MHG-Mem with eight memory systems across three benchmarks and three backbones. Figure~\ref{exp:memory_tasks} shows memory quality on WorldMemArena-S and fine-grained task performance on WorldMemArena-S and Mem-Gallery.
We highlight four key findings:

\textbf{(1) MHG-Mem consistently outperforms existing memory systems across benchmarks and backbones.}
MHG-Mem achieves the highest Judge score in every benchmark--backbone setting, with Avg.\ J.\ scores of 57.98\% (\textbf{+8.20\%}), 67.95\% (\textbf{+9.89\%}), and 73.28\% (\textbf{+9.98\%}) on Qwen3-VL-8B, GPT-4.1-mini, and Qwen3.8-27B, respectively.
Moreover, MHG-Mem achieves the best or second-best results on metrics including EM, BLEU-1, and F1. These consistent improvements across models and benchmarks demonstrate the effectiveness and generalization of provenance-preserving multimodal hypergraph memory.

\textbf{(2) Effective visual memory requires structured access to preserved observations.}
On the visual-centric MemEye benchmark, MHG-Mem with Qwen3.8-27B achieves 78.76\% EM, compared with 67.85\% for UniversalRAG and 66.95\% for SRAG.
A similar advantage appears in visually demanding categories in Figure~\ref{exp:memory_tasks}(b): MHG-Mem achieves 82\% versus 60\% on visual reasoning (VR) in Mem-Gallery and 87\% versus 83\% on visual fact recall (VFR) in WorldMemArena-S.
These results suggest that retaining images alone is insufficient when relevant observations cannot be reliably associated with the underlying entity.
Occurrence hyperedges address this limitation by grouping visually diverse mentions of the same referent while preserving direct access to observation-specific visual evidence.

\textbf{(3) Substantial performance gains on temporal and cross-modal reasoning.}
MHG-Mem achieves the best performance in 17 of the 19 task categories in Figure~\ref{exp:memory_tasks}(b), with particularly large gains on tasks requiring integration of observations across sessions.
On WorldMemArena-S, temporal reasoning (TR) improves from 48\% to 87\% over the strongest baseline, while dynamic update (DU), test-time learning (TTL), and cross-modal reasoning (CMR) improve by 17\%, 17\%, and 15\%, respectively.
On Mem-Gallery, visual reasoning (VR) and temporal reasoning (TR) improve by 18\% and 17\%.
These gains are supported by our evolution design, which summarize time-ordered entity states to guide retrieval, together with evidential hyperedges that preserve state--fact correspondence and enable provenance tracing from synthesized states back to the originating dialogue and images.

\textbf{(4) Memory quality and evidence accessibility improve together.}
As shown in Figure~\ref{exp:memory_tasks}(a), MHG-Mem achieves the highest memory recall, memory correctness, and retrieval coverage among all compared systems.
The simultaneous improvement in memory quality and coverage indicates that MHG-Mem preserves useful information accurately while keeping its supporting evidence readily accessible.

\begin{figure}[t]
\centering
\setlength{\abovecaptionskip}{0pt} 
\setlength{\belowcaptionskip}{0pt} 
\includegraphics[
    width=\linewidth,
    trim={0 0 0 0},
    clip
]{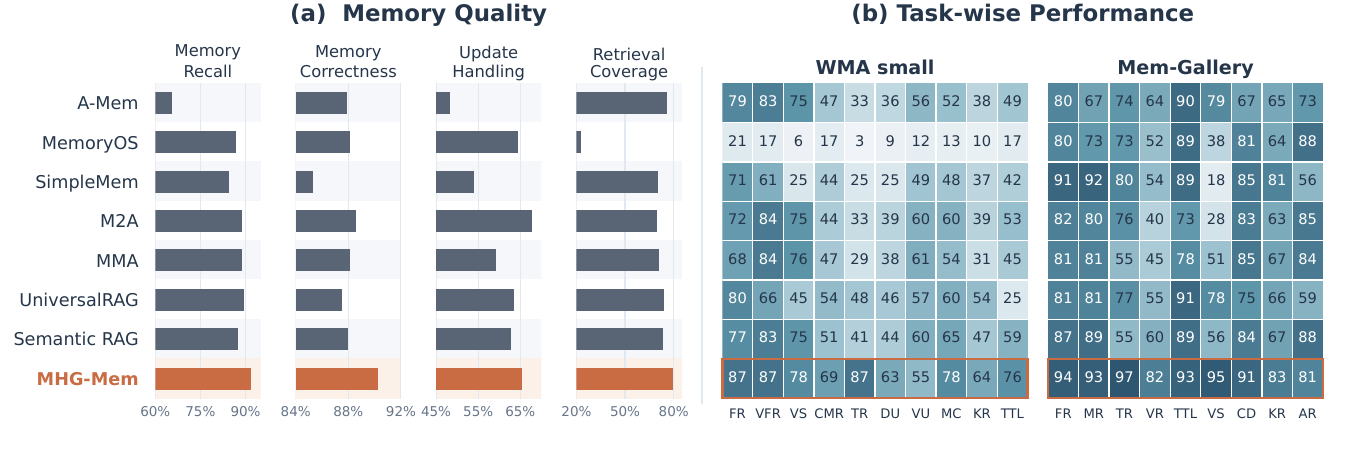}
\caption{(a) Memory quality and (b) fine-grained tasks performance on WorldMemArena-S (WMA-S) and Mem-Gallery with Qwen3.8-27B.}
\label{exp:memory_tasks}
\end{figure}

\subsection{Ablation Study}
\label{sec:ablation}

Tables~\ref{tab:mhg-ablation} and~\ref{tab:mhg-search-ablation} evaluate memory construction and retrieval components using the average LLM-judge score across the three benchmarks.
Memory components are introduced cumulatively, whereas retrieval components are removed individually. Specifically, \textbf{w/o Agentic Search} replaces agentic search with three-layer top-$k$ retrieval; \textbf{w/o Inspection Tool}, \textbf{w/o Cross-layer Traversal}, and \textbf{w/o Visual Tool} disable the corresponding functionality.

\paragraph{Memory-side ablation.}
\textbf{(1) Effectiveness of the entity node with occurrence hyperedge.}
The entity layer with occurrence hyperedges improves performance by $5.48\%$ over fact-only memory. Occurrence hyperedges account for $4.78\%$ of this gain, with improvements across all three benchmarks. These hyperedges connect mentions and facts associated with the same entity, enabling retrieval of its observations across sessions.
\textbf{(2) Effectiveness of the evolution node with evolutionary hyperedge.}
The evolution layer with evolutionary hyperedges yields a further $3.37\%$ improvement. This representation makes state histories directly retrievable and associates them with entities in subject and coactor roles, supporting reasoning about changes across sessions.
\textbf{(3) Effectiveness of the evidential hyperedge.}
Evidential hyperedges provide an additional $1.56\%$ gain by linking evolution states to their supporting facts. These connections preserve provenance and allow retrieved states to be checked against the observations from which they were derived.

\paragraph{Retrieval-side ablation.}
\textbf{(1) Effectiveness of agentic search.}
Three-layer top-$k$ retrieval underperforms agentic search by $5.10\%$, the largest decline among the retrieval ablations. This result supports query-adaptive selection of retrieval entry points and subsequent evidence expansion.
\textbf{(2) Effectiveness of the inspection tool.}
Without source inspection, performance falls by $3.44\%$, the largest decline among individual tool ablations. Access to original dialogue and images enables the agent to verify fine-grained evidence that facts and evolution summaries may omit.
\textbf{(3) Effectiveness of cross-layer traversal.}
Removing cross-layer traversal reduces performance by $3.09\%$. Traversal connects entities to their associated facts and evolution records to supporting observations, allowing the agent to assemble evidence across memory granularities.
\textbf{(4) Effectiveness of the visual tool.}
Performance decreases by $1.37\%$ when image retrieval is disabled. Direct visual search provides a complementary route to relevant observations, extending evidence access beyond entity- and evolution-centred retrieval.


\begin{table}[t]
    \centering
    \setlength{\parskip}{0pt}
    \renewcommand{\footnotesize}{\fontsize{7}{8}\selectfont}
    \captionsetup{font=footnotesize,labelfont=normalfont,textfont=normalfont,
        justification=raggedright,singlelinecheck=false,position=top,skip=2pt}
    \definecolor{mhgheadergray}{gray}{0.93}
    \setlength{\tabcolsep}{1.8pt}
    \renewcommand{\arraystretch}{1.05}
    \begin{minipage}[t]{0.49\linewidth}
        \vspace{0pt}
        \centering
        \caption{Ablation of memory side with Qwen3.8-27B.}
        \label{tab:mhg-ablation}
        \label{tab:mhg-layer-ablation}
        \fontsize{7}{8}\selectfont
        \begin{adjustbox}{max width=\linewidth}
            \begin{tabular}{@{}lc||cccc@{}}
                \noalign{\hrule height 0.6pt}
                \rowcolor{mhgheadergray}
                \textbf{Layers} & \textbf{Hyperedges}
                    & \textbf{MemEye}
                    & \textbf{Mem-Gallery}
                    & \textbf{WMA-S}
                    & \textbf{Avg.} \\
                \hline\hline
                $\mathcal{F}$      &     & 53.81 & 76.08 & 58.72 & 62.87 \\
                $\mathcal{F}\,\dot\cup\,\mathcal{U}$    &     & 54.06 & 77.54 & 59.10 & 63.57 \\
                $\mathcal{F}\,\dot\cup\,\mathcal{U}$    & $\mathcal{E}_{\mathrm{occ}}$   & 60.69 & 82.13 & 62.24 & 68.35 \\
                $\mathcal{F}\,\dot\cup\,\mathcal{U}\,\dot\cup\,\mathcal{Z}$ & $\mathcal{E}_{\mathrm{occ}}\,\dot\cup\,\mathcal{E}_{\mathrm{evo}}$ & 64.70 & 85.48 & 64.97 & 71.72 \\
                \arrayrulecolor{gray}
                \midrule
                \arrayrulecolor{black}
                \textbf{$\mathcal{F}\,\dot\cup\,\mathcal{U}\,\dot\cup\,\mathcal{Z}$} & \textbf{$\mathcal{E}_{\mathrm{occ}}\,\dot\cup\,\mathcal{E}_{\mathrm{evo}}\,\dot\cup\,\mathcal{E}_{\mathrm{evd}}$}
                    & \textbf{65.93} & \textbf{87.85} & \textbf{66.05} & \textbf{73.28} \\
                \hline
            \end{tabular}
        \end{adjustbox}
    \end{minipage}\hfill%
    \begin{minipage}[t]{0.49\linewidth}
        \vspace{0pt}
        \centering
        \caption{Ablation of retrieval side with Qwen3.8-27B.}
        \label{tab:mhg-search-ablation}
        \fontsize{7}{8}\selectfont
        \begin{adjustbox}{max width=\linewidth}
            \begin{tabular}{@{}l||cccc@{}}
                \noalign{\hrule height 0.6pt}
                \rowcolor{mhgheadergray}
                \textbf{Method}
                    & \textbf{MemEye}
                    & \textbf{Mem-Gallery}
                    & \textbf{WMA-S}
                    & \textbf{Avg.} \\
                \hline\hline
                w/o Agentic Search
                    & 60.59 & 82.06 & 61.90 & 68.18 \\
                w/o Inspection Tool
                    & 61.77 & 84.64 & 63.12 & 69.84 \\
                w/o Cross-layer Traversal
                    & 62.33 & 84.49 & 63.74 & 70.19\\
                w/o Visual Tool
                    & 63.52 & 86.98 & 65.23 & 71.91\\
                \arrayrulecolor{gray}
                \midrule
                \arrayrulecolor{black}
                \textbf{ MHG-Mem (Full)}
                    & \textbf{65.93} & \textbf{87.85} & \textbf{66.05} & \textbf{73.28} \\
                \hline
            \end{tabular}
        \end{adjustbox}
    \end{minipage}
\end{table}

\begin{figure}[t]
\centering
\includegraphics[width=0.9\linewidth]{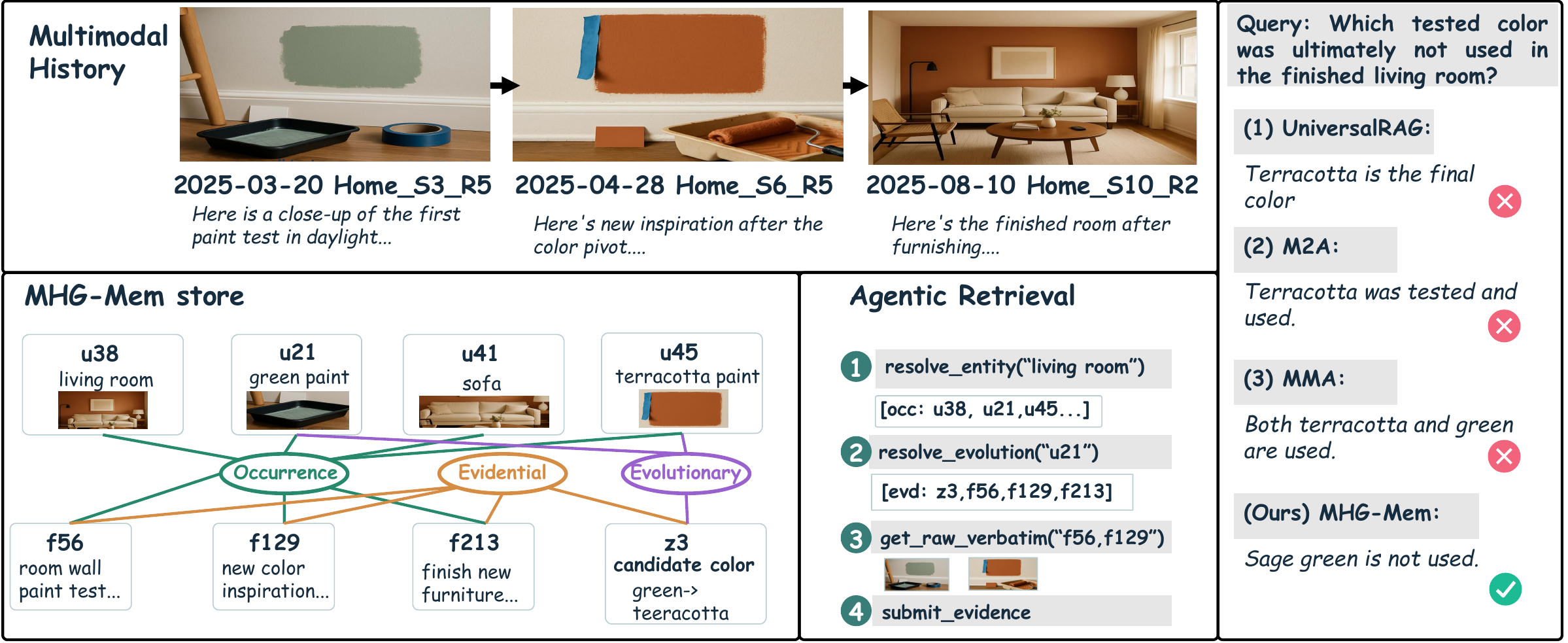}
\caption{Case study in a home renovation scenario. MHG-Mem tracks wall paint color changes through evolution nodes and traces supporting facts to source images via evidential hyperedges, forming a verifiable evidence chain.}
\label{fig:case1}
\vspace{-10pt}
\end{figure}

\subsection{Case Analysis}
\label{sec:case_study}
Figure~\ref{fig:case1} illustrates the advantages of MHG-Mem in a home renovation scenario involving successive living room color schemes. To identify the scheme the user did not adopt, the memory system must trace the progression from color trials through revisions to completion and inspect the corresponding images. UniversalRAG and M2A retrieve bedroom images by similarity but miss the initial color information, while MMA preserves the pixels without tracking the evolving schemes. MHG-Mem identifies \texttt{paint} as the tracked entity, preserves its mentions across rounds, and links atomic facts through occurrence hyperedges. Evolution nodes record the change in candidate color from sage green to terracotta. During retrieval, the agent follows evidential hyperedges to \texttt{f56} and \texttt{f129}, then inspects their source images to verify the colors and answer correctly. 

\section*{Conclusion}
\label{sec:conclusion}

We presented MHG-Mem, a multimodal hypergraph memory framework that organizes visually diverse observations of the same entity into coherent state histories while preserving fine-grained visual evidence. Typed n-ary hyperedges connect entity mentions, atomic facts, and evolution records, keeping individual states traceable to their supporting observations. Hypergraph-based agentic search adaptively navigates these memory levels, using state histories to guide retrieval and original dialogue and images to verify evidence. Ablations demonstrate the contributions of entity associations, explicit state histories, evidence grounding, and adaptive retrieval. These findings support the joint design of memory organization and evidence access for reasoning over evolving multimodal experience.

\bibliography{iclr2027_conference}
\bibliographystyle{iclr2027_conference}

\clearpage
\appendix
\startcontents[appendix]

\section*{Appendix Contents}
\printcontents[appendix]{}{1}{%
    \setcounter{tocdepth}{2}%
}

\clearpage

\ifcsname mhgproposition\endcsname
\else
    \newtheorem{mhgproposition}{Proposition}
\fi

\section{Notation Definition}
\label{app:notation}

\begin{table}[H]
    \centering
    \footnotesize
    \setlength{\tabcolsep}{4pt}
    \renewcommand{\arraystretch}{1.08}
    \label{tab:mhg-symbols}
    \begin{tabularx}{\linewidth}{@{}>{$}l<{$}>{\raggedright\arraybackslash}X>{$}l<{$}>{\raggedright\arraybackslash}X@{}}
        \toprule
        Symbol & Meaning & Symbol & Meaning \\
        \midrule
        \mathcal{D}_{\leq T} & Ordered multimodal history
            & d_t & Interaction round $t$ \\
        x_t & Dialogue text in round $t$
            & \mathcal{I}_t & Images in round $t$ \\
        \delta_t & Date of round $t$
            & \mathcal{C}_t & Recent within-session context \\
        q=(x_q,\mathcal{I}_q) & Multimodal query
            & \mathcal{W},\mathcal{S} & Memory construction and retrieval operators \\
        E_q,\hat{y} & Retrieved evidence and generated answer
            & \theta & Fixed model parameters \\
        \mathcal{H} & Attributed multimodal hypergraph
            & \mathcal{M}_T & Persistent memory at history length $T$ \\
        \mathcal{V},\mathcal{E} & Node and hyperedge sets
            & \iota & Hyperedge incidence map \\
        \tau_V,\tau_E & Node-type and hyperedge-type maps
            & \rho & Source-round reference map \\
        \mathcal{F},\mathcal{U},\mathcal{Z} & FactNode, EntityNode, and EvolutionNode sets
            & \mathcal{K} & Stored embedding indices \\
        \mathcal{E}_{\mathrm{occ}},\mathcal{E}_{\mathrm{evo}},\mathcal{E}_{\mathrm{evd}}
            & Occurrence, evolutionary, and evidential hyperedge families
            & r & Referent group \\
        f=(\ell_f,\mathcal{I}_f,t_f,s_f,\delta_f)
            & FactNode and its text, images, round, session, and date
            & u=(n_u,c_u,\mu_u,i_u,b_u) & EntityNode and its mention attributes \\
        \mathcal{U}_r,\mathcal{F}_r & Mentions and facts associated with referent $r$
            & e_r^{\mathrm{occ}} & Occurrence hyperedge of $r$ \\
        z=(r_z,a_z,\sigma_z,S_z,\kappa_z)
            & EvolutionNode and its subject, aspect, summary, state trace, and update type
            & \mathcal{F}_z & Facts supporting evolution $z$ \\
        S_z=((v_j,f_j))_{j=1}^{m_z} & Ordered state--fact trace of $z$
            & v_j,m_z & State value at position $j$ and number of recorded states \\
        \mathcal{U}_z^{\mathrm{subj}},\mathcal{U}_z^{\mathrm{co}}
            & Subject and coactor mentions of $z$
            & e_z^{\mathrm{evo}},e_z^{\mathrm{evd}} & Evolutionary and evidential hyperedges \\
        a_k,o_k,h_k & Tool action, observation, and search history at step $k$
            & \pi_\theta & LLM search controller \\
        B_{\mathrm{step}},B_{\mathrm{img}} & Search-step and image budgets
            & w & Controller-generated fact-search string \\
        \mathcal{F}_t,\widehat{\mathcal{Z}}_r & Current-round facts and candidate evolutions
            & A & Fact-to-occurrence anchoring predicate \\
        s_T,s_{TI},s_{II};\lambda_T,\lambda_{TI},\lambda_{II}
            & Fact-level component scores and weights
            & w_T,w_I,c_T,c_I & Entity-linking weights and similarities \\
        s_F,s_Z & Fact- and evolution-level relevance scores
            & \phi_T,\phi_I,\psi_T & Text, image, and cross-modal encoders \\
        \operatorname{Traverse}(v,\mathcal{E}',\mathcal{V}') & Cross-layer traversal operator
            & R_z & Source rounds cited by evolution $z$ \\
        \mathcal{R},\operatorname{Src}(z) & Retained archive and recorded source set of $z$
            & K & Candidate limit in construction \\
        \bottomrule
    \end{tabularx}
\end{table}

\section{Additional Method Details}
\label{app:mhg-method}

This section provides more detailed hypergraph schema, notation definition, construction and search algorithms, tool specifications, and a source-recoverability proof.

\subsection{Detailed Hypergraph Memory Schema}
\label{app:mhg-schema}

MHG-Mem is an attributed hypergraph with three node types and three typed
$n$-ary hyperedge families:

\paragraph{Nodes.}
\textbf{FactNode} is an atomic multimodal observation such as an event,
state, visual attribute, or spatial relation. Each $f\in\mathcal{F}$ stores
\begin{equation}
    f=(\ell_f,\mathcal{I}_f,t_f,s_f,\delta_f),
    \label{eq:mhg-fact-node}
\end{equation}
where $\ell_f$ is the fact description, $\mathcal{I}_f$ is its set of image
identifiers, and $(t_f,s_f,\delta_f)$ gives the source round, session, and
date.

\textbf{EntityNode} represents a distinct entity mention expressed as a text span, an image region, or both. Each
$u\in\mathcal{U}$ stores
\begin{equation}
    u=(n_u,c_u,\mu_u,i_u,b_u),
    \qquad b_u=(x_1,y_1,x_2,y_2)\in[0,1]^4,
    \label{eq:mhg-entity-node}
\end{equation}
where $n_u$ is its name, $c_u$ its entity class, $\mu_u$ its modality,
$i_u$ an optional image identifier, and $b_u$ an optional normalized visual
region. EntityNodes linked to the same underlying entity form a referent
group $r$ with mention set $\mathcal{U}_r$; every mention remains a distinct
observation.

\textbf{EvolutionNode} records a change in an attribute or relation of a
referent group, such as changes in appearance, price, or location. Each $z\in\mathcal{Z}$ stores
\begin{equation}
    z=(r_z,a_z,\sigma_z,S_z,\kappa_z),
    \qquad
    S_z=((v_j,f_j))_{j=1}^{m_z},
    \qquad
    \mathcal{F}_z=\{f_j:1\leq j\leq m_z\}\subseteq\mathcal{F}_{r_z},
    \label{eq:mhg-evolution-node}
\end{equation}
where $r_z$ is the subject referent, $a_z$ is the evolving attribute or
relation, $\sigma_z$ summarizes the observed states, $S_z$ is the ordered
state--fact trace, and $\kappa_z$ is the update type.

\paragraph{Hyperedges.}
\textbf{Occurrence hyperedge} binds all mentions and associated facts of a
referent group:
\begin{equation}
    \iota(e_r^{\mathrm{occ}})=\mathcal{U}_r\cup\mathcal{F}_r,
    \qquad
    \mathcal{U}_r=\iota(e_r^{\mathrm{occ}})\cap\mathcal{U}.
    \label{eq:mhg-occurrence}
\end{equation}

\textbf{Evolutionary hyperedge} binds an EvolutionNode to its subject and
coactor mentions:
\begin{equation}
    \iota(e_z^{\mathrm{evo}})=
    \{z\}\cup\mathcal{U}_z^{\mathrm{subj}}\cup\mathcal{U}_z^{\mathrm{co}},
    \label{eq:mhg-evolutionary}
\end{equation}
where $\mathcal{U}_z^{\mathrm{subj}}\subseteq\mathcal{U}_{r_z}$ contains
mentions of the changing entity and $\mathcal{U}_z^{\mathrm{co}}$ contains
mentions from other participating referents. 

\textbf{Evidential hyperedge}
connects the evolution to its supporting facts:
\begin{equation}
    \iota(e_z^{\mathrm{evd}})=\{z\}\cup\mathcal{F}_z.
    \label{eq:mhg-evidential}
\end{equation}

Finally, $\rho(v)$ gives the source-round reference for
$v\in\mathcal{F}\cup\mathcal{U}$, while $\mathcal{I}_f$ and $(i_u,b_u)$
retain image-level and region-level grounding. The complete memory is
\begin{equation}
    \mathcal{M}_T=(\mathcal{H},\rho,\mathcal{K}).
    \label{eq:mhg-memory}
\end{equation}

\subsection{Detailed Memory Construction}
\label{app:mhg-construction-details}

Algorithm~\ref{alg:mhg-build} combines the three construction steps.
Original rounds and images are archived before extraction. FactNodes
retain round-level image references $\mathcal{I}_f$, whereas visual
EntityNodes retain observation-specific regions $(i_u,b_u)$; both use
$\rho$ to locate their source rounds.


\paragraph{Evolution validation and participants.}
The synthesizer reads ordered fact descriptions and their identifiers,
and proposes histories for changing attributes or relations across
sessions. Invalid fact references and repeated references are removed;
the remaining state--fact pairs are ordered by $t_f$.
For a candidate with subject $r_z=r$, retention requires
\begin{equation}
    \begin{aligned}
        &f_j\in\mathcal{F}_r &&(1\leq j\leq m_z),\\
        &t_{f_j}\leq t_{f_{j+1}} &&(1\leq j<m_z),\\
        &\bigl|\{v_j:1\leq j\leq m_z\}\bigr|\geq 2.
    \end{aligned}
    \label{eq:app-evolution-validity}
\end{equation}
The values must be nonempty, and distinctness is checked after case and
whitespace normalization. These checks validate the recorded sequence;
identifying a meaningful change remains the synthesizer's task.

Let $R_z=\{\rho(f):f\in\mathcal{F}_z\}$ be the cited round references.
Subject mentions are selected from $\mathcal{U}_{r_z}$ with
$\rho(u)\in R_z$; if none occurs in these rounds, one stored subject
mention identifies the referent. For each other group $r'\ne r_z$ with
$\mathcal{F}_{r'}\cap\mathcal{F}_z\ne\varnothing$, a mention satisfying
$\rho(u)\in R_z$ is selected as a coactor when available.
These role sets define $e_z^{\mathrm{evo}}$; the distinct state references
define $e_z^{\mathrm{evd}}$ through
Equations~\eqref{eq:mhg-evolutionary} and~\eqref{eq:mhg-evidential}.
The ordered correspondence remains in $S_z$.

\begin{algorithm}[t]
    \caption{MHG-Mem construction}
    \label{alg:mhg-build}
    \begin{algorithmic}[1]
        \Require History $\mathcal{D}_{\leq T}$, construction models,
                 candidate limit $K$
        \Ensure Memory $\mathcal{M}_T=(\mathcal{H},\rho,\mathcal{K})$
        \State Initialize the archive, $\mathcal{H}$, $\rho$, and $\mathcal{K}$
        \For{$t=1,\ldots,T$}
            \State Archive $d_t$; index its images and captions
            \State Extract $\mathcal{F}_t$ using Equation~\eqref{eq:mhg-fact-extraction}
                   \Comment{\textbf{Step 1}}
            \State Store $\mathcal{F}_t$, source references, and text embeddings
            \State Extract entity candidates from $d_t$ \Comment{\textbf{Step 2}}
            \For{each candidate $u$}
                \State Rank occurrence hyperedges by Equation~\eqref{eq:mhg-link-score}
                \State Condition on $d_t$, $u$, and the top-$K$ candidates
                \State Select \texttt{UPDATE}, \texttt{ADD}, or \texttt{SKIP}
                \If{the action is not \texttt{SKIP}}
                    \State Store $u$ and $\rho(u)$; execute the occurrence update
                    \State Index its name and available visual region
                \EndIf
            \EndFor
            \State Anchor $\mathcal{F}_t$ using Equation~\eqref{eq:mhg-occurrence-update}
        \EndFor
        \For{each referent group $r$ with facts across sessions}
            \State Generate $\widehat{\mathcal{Z}}_r$ using
                   Equation~\eqref{eq:mhg-evolution-synthesis}
                   \Comment{\textbf{Step 3}}
            \State Filter and order state references; apply
                   Equation~\eqref{eq:app-evolution-validity}
            \For{each retained $z\in\widehat{\mathcal{Z}}_r$}
                \State Set $r_z=r$; construct subject and coactor memberships
                \State Store $z$, $e_z^{\mathrm{evo}}$, and $e_z^{\mathrm{evd}}$
                \State Index $\phi_T(a_z\mathbin{\|}\sigma_z)$ in $\mathcal{K}$
            \EndFor
        \EndFor
        \State \Return $\mathcal{M}_T$
    \end{algorithmic}
\end{algorithm}

\paragraph{Incremental construction.}
At a new session boundary, only groups receiving newly anchored facts
require resynthesis. Their evolution records are regenerated from the
updated $\mathcal{F}_r$. Successful replacements update the derived nodes,
hyperedges, and embeddings while retaining the underlying facts, mentions,
and original sources.

\subsection{Agentic Search and Tool Specification}
\label{app:mhg-search-details}

Table~\ref{tab:mhg-tools} specifies the six tools available to the controller.
Algorithm~\ref{alg:mhg-search} uses them under a step budget
$B_{\mathrm{step}}$ and a budget $B_{\mathrm{img}}$ on distinct images
presented during search.

\begin{table}[t]
    \centering
    \small
    \caption{Four tool types for hypergraph-based agentic search.
    Query images, when supplied, are available to the visual lookups.}
    \label{tab:mhg-tools}
    \begin{tabularx}{\linewidth}{@{}>{\raggedright\arraybackslash}p{0.14\linewidth}>{\raggedright\arraybackslash}p{0.22\linewidth}>{\raggedright\arraybackslash}p{0.29\linewidth}>{\raggedright\arraybackslash}X@{}}
        \toprule
        Type & Tool & Input & Returned information \\
        \midrule
        Hypergraph
            & \texttt{resolve\_facts}
            & Text query $w$
            & Ranked facts and source references. \\
            & \texttt{resolve\_entity}
            & Entity name or query image
            & Candidate referent groups and corresponding occurrence linking facts. \\
            & \texttt{resolve\_evolution}
            & Referent name, query image, or \texttt{evolution\_id}
            & Evolutions and evidential linking facts. \\
        Visual
            & \texttt{search\_image}
            & Text query or query image
            & Matching images. \\
        Inspection
            & \texttt{get\_raw\_verbatim}
            & Source-round identifier
            & Original dialogue and images. \\
        Assembly
            & \texttt{submit\_evidence}
            & Fact and image references; optional summary
            & Validation feedback; accepted evidence forms $E_q$. \\
        \bottomrule
    \end{tabularx}
\end{table}

\paragraph{Multimodal embedding scoring.}
Using the text encoder $\phi_T$ from the main text, let $\phi_I$ encode
images and $\psi_T$ encode text in $\phi_I$'s feature space.
An image identifier is resolved to its pixels before encoding.
The components of Equation~\eqref{eq:mhg-fact-score} are
\begin{equation}
    \begin{aligned}
        s_T(w,f)&=\cos\bigl(\phi_T(w),\phi_T(\ell_f)\bigr),\\
        s_{TI}(w,f)&=\max\bigl(\{0\}\cup
            \{\cos(\psi_T(w),\phi_I(I)):I\in\mathcal{I}_f\}\bigr),\\
        s_{II}(q,f)&=\max\bigl(\{0\}\cup
            \{\cos(\phi_I(J),\phi_I(I)):
                J\in\mathcal{I}_q,\ I\in\mathcal{I}_f\}\bigr).
    \end{aligned}
    \label{eq:app-fact-score-components}
\end{equation}
Visual scores are floored at zero, and unavailable terms contribute zero
without renormalizing $\lambda_T,\lambda_{TI},\lambda_{II}$.
The Visual Tool searches stored captions for a text lookup and image
embeddings for a query-image lookup.

\paragraph{Traversal and evidence assembly.}
Entity and evolution lookups perform the membership expansions described
in Section~\ref{sec:method-search}. Evolution relevance uses
$a_z\mathbin{\|}\sigma_z$, while readout dereferences the supporting facts
in the order retained by $S_z$. Source inspection follows $\rho(f)$ to
recover original observations. Facts returned by retrieval or associated
with accessed rounds are registered for citation. At submission, fact
identifiers must be registered and available, and attached image identifiers
must belong to the cited fact's source round. Invalid references are removed.
The retained references select textual and visual evidence for $E_q$,
together with the optional controller summary. These checks validate
identifier associations, not the semantic content of that summary.


\subsection{Provenance Preservation}
\label{app:mhg-properties}

Let $\mathcal{R}$ be the archive mapping source-round identifiers to
retained observations, so $\rho(f)\in\operatorname{dom}(\mathcal{R})$
means that the source of $f$ is available.
For $S_z=((v_j,f_j))_{j=1}^{m_z}$, define the recorded source set
\begin{equation}
    \operatorname{Src}(z)=\{\rho(f_j):1\leq j\leq m_z\}.
    \label{eq:app-source-set}
\end{equation}
The following property concerns access to recorded evidence, independently
of extraction accuracy or the truth of generated summaries.

\begin{mhgproposition}[Source recoverability]
\label{prop:mhg-recoverability}
Suppose every retained state cites $f_j\in\mathcal{F}_{r_z}$ with
$\rho(f_j)\in\operatorname{dom}(\mathcal{R})$, and
$e_z^{\mathrm{evd}}$ stores exactly the distinct facts cited in $S_z$ as
in Equation~\eqref{eq:mhg-evidential}. Then each state has a trace to its
original observation, and expanding $e_z^{\mathrm{evd}}$ followed by
$\rho$ recovers exactly $\operatorname{Src}(z)$. This property is
preserved by fact and mention additions and validated evolution
insertions or replacements with valid fact identifiers, source bindings,
and archive records.
\end{mhgproposition}

\begin{proof}
Because $\mathcal{F}$ and $\mathcal{Z}$ are disjoint,
$z\notin\mathcal{F}$. Since every $f_j\in\mathcal{F}_{r_z}\subseteq
\mathcal{F}$, Equation~\eqref{eq:mhg-traverse} gives
\begin{equation}
    \begin{aligned}
        \operatorname{Traverse}(z,\{e_z^{\mathrm{evd}}\},\mathcal{F})
            &=\iota(e_z^{\mathrm{evd}})\cap\mathcal{F}\\
            &=(\{z\}\cup\mathcal{F}_z)\cap\mathcal{F}
              =\mathcal{F}_z.
    \end{aligned}
    \label{eq:app-provenance-membership}
\end{equation}
For every $1\leq j\leq m_z$, the membership $f_j\in\mathcal{F}_z$
and the source-availability assumption therefore define the trace
\begin{equation}
    z\xrightarrow{\ e_z^{\mathrm{evd}}\ }f_j
      \xrightarrow{\ \rho\ }\rho(f_j)
      \xrightarrow{\ \mathcal{R}\ }\mathcal{R}(\rho(f_j)).
    \label{eq:app-provenance-path}
\end{equation}
Applying $\rho$ to all recovered fact members yields
\begin{equation}
    \begin{aligned}
        &\{\rho(f):f\in
            \operatorname{Traverse}(z,\{e_z^{\mathrm{evd}}\},\mathcal{F})\}\\
        &\qquad=\{\rho(f):f\in\mathcal{F}_z\}\\
        &\qquad=\{\rho(f_j):1\leq j\leq m_z\}
                 =\operatorname{Src}(z).
    \end{aligned}
    \label{eq:app-provenance-recovery}
\end{equation}
Thus expansion neither adds nor omits a recorded source. The set equality
deduplicates shared sources; the index-specific association
$j\mapsto f_j\mapsto\rho(f_j)$ remains available through $S_z$.

For preservation, denote the evolution sets before and after an update by
$\mathcal{Z}^{-}$ and $\mathcal{Z}^{+}$. If $\mathcal{Z}_{\mathrm{del}}$
contains removed evolutions and $\mathcal{Z}_{\mathrm{new}}$ contains
validated new evolutions, then
\begin{equation}
    \mathcal{Z}^{+}
      =(\mathcal{Z}^{-}\setminus\mathcal{Z}_{\mathrm{del}})
          \,\dot\cup\,\mathcal{Z}_{\mathrm{new}}.
    \label{eq:app-evolution-update}
\end{equation}
For every surviving $z$, unchanged state references and source bindings give
\begin{equation}
    \operatorname{Src}^{+}(z)=\operatorname{Src}^{-}(z)
       \subseteq\operatorname{dom}(\mathcal{R}^{-})
       \subseteq\operatorname{dom}(\mathcal{R}^{+}),
    \label{eq:app-source-preservation}
\end{equation}
where superscripts identify the corresponding memory snapshots.
Its evidential membership is unchanged, so
Equations~\eqref{eq:app-provenance-membership}--\eqref{eq:app-provenance-recovery}
continue to hold. For every new evolution, validation supplies the same
assumptions, and the preceding argument applies directly.
Fact and mention additions are the case
$\mathcal{Z}_{\mathrm{del}}=\mathcal{Z}_{\mathrm{new}}=\varnothing$.
The empty memory satisfies the property vacuously; induction over these
operations establishes it for every successfully constructed snapshot.
\end{proof}

\begin{algorithm}[t]
    \caption{Hypergraph-based agentic search}
    \label{alg:mhg-search}
    \begin{algorithmic}[1]
        \Require Query $q$, memory $\mathcal{M}_T$, controller $\pi_\theta$,
                 step budget $B_{\mathrm{step}}$, image budget $B_{\mathrm{img}}$
        \Ensure Retrieved evidence $E_q$
        \State Initialize $h_1=()$ and an empty registry of retrieved facts
        \For{$k=1,\ldots,B_{\mathrm{step}}$}
            \State Reason over $q,h_k$; select $a_k\sim\pi_\theta(\,\cdot\mid q,h_k)$
            \If{$a_k$ selects hypergraph retrieval}
                \State $o_k\gets$ layer-specific retrieval with applicable traversal
                       (Eq.~\eqref{eq:mhg-traverse})
            \ElsIf{$a_k$ selects visual retrieval}
                \State $o_k\gets$ matching images and captions
            \ElsIf{$a_k$ selects source inspection}
                \State $o_k\gets$ original dialogue and images via $\rho$
            \Else \Comment{Evidence submission}
                \State Validate submitted references and source associations
                \If{submission is accepted}
                    \State \Return $E_q$ assembled from selected evidence and notes
                \EndIf
                \State $o_k\gets$ feedback requesting additional evidence
            \EndIf
            \State Register retrieved facts and facts associated with accessed sources
            \State Attach requested images to $o_k$ within $B_{\mathrm{img}}$
            \State $h_{k+1}\gets h_k\mathbin{\|}(a_k,o_k)$
        \EndFor
        \State \Return Finalize $E_q$ through forced evidence submission
    \end{algorithmic}
\end{algorithm}

\section{Additional Experimental Details}
\label{app:experiment}

\subsection{Benchmarks and Metrics}
\label{app:experiment-benchmarks}

\paragraph{MemEye.}
MemEye~\citep{guo2026memeye} evaluates visual memory along two axes:
visual evidence granularity, from scene-level to pixel-level details,
and memory reasoning depth, from atomic retrieval to evolutionary
synthesis. It contains eight everyday scenarios and 371 questions, each
with paired multiple-choice and open-ended versions. These tasks require
agents to recover visually grounded facts, distinguish similar instances,
and reason over observations accumulated across sessions.

\paragraph{Mem-Gallery.}
Mem-Gallery~\citep{bei2026memgallery} evaluates long-term multimodal
conversational memory using persona-driven dialogues across 20 scenarios.
The evaluated release contains 240 sessions, 3,962 dialogue rounds,
1,490 images, and 1,711 open-ended questions. It probes memory extraction,
reasoning, and knowledge management over extended text--image histories.

\paragraph{WorldMemArena-S.}
WorldMemArena~\citep{liu2026worldmemarena} evaluates memory in GUI and
embodied-agent trajectories, as well as lifelong personal and project
interactions. We use its released 150-sample subset, denoted
WorldMemArena-S, containing 7,906 questions. Its 11 question types cover
textual memory, visual memory, and reasoning, including factual recall,
dynamic updates, conflict resolution, and cross-modal reasoning.
Evaluation follows the annotated QA checkpoints: MHG-Mem constructs
memory from only the sessions covered by each checkpoint. Prefix
checkpoints reuse the corresponding historical state, while non-prefix
session windows are constructed independently to prevent access to
out-of-scope observations.

\paragraph{Evaluation Metrics.}
Let $\mathcal{Q}_b$ denote the open-ended questions of benchmark $b$ and
$\mathcal{Q}_{\mathrm{MemEye}}^{\mathrm{MCQ}}$ its multiple-choice questions.
Answers are lowercased, punctuation-normalized, and whitespace-normalized
before scoring; all reported values are percentages.

For MemEye, options are rotated across $R_q$ equivalent instances. Exact Match
(EM) is the mean accuracy over rotations and questions:
\begin{equation}
\mathrm{EM}=\frac{1}{|\mathcal{Q}_{\mathrm{MemEye}}^{\mathrm{MCQ}}|}
\sum_{q}\frac{1}{R_q}\sum_{r=1}^{R_q}
\mathbf{1}[\hat{c}_{q,r}=c^\star_{q,r}].    
\end{equation}

For open-ended answers, let $\widetilde{T}(\cdot)$ be the normalized token
multiset after Porter stemming and $C_q=\widetilde{T}(\hat{y}_q)\cap
\widetilde{T}(y_q^\star)$. Token F1 is
\begin{equation}
    F_{1,q}=\frac{2P_qR_q}{P_q+R_q},\quad
P_q=\frac{|C_q|}{|\widetilde{T}(\hat{y}_q)|},\quad
R_q=\frac{|C_q|}{|\widetilde{T}(y_q^\star)|},
\qquad
\mathrm{F1}_b=\frac{1}{|\mathcal{Q}_b|}\sum_{q\in\mathcal{Q}_b}F_{1,q}
\end{equation}

with $F_{1,q}=0$ when the overlap is empty. BLEU-1 uses the same
normalization without stemming and combines clipped unigram precision with a
brevity penalty:
\begin{equation}
    \mathrm{BLEU1}_q=\mathrm{BP}_q\,\frac{\sum_w\min\{c_{\hat{y}_q}(w),c_{y_q^\star}(w)\}}
{\sum_w c_{\hat{y}_q}(w)},\qquad
\mathrm{B\text{-}1}_b=\frac{1}{|\mathcal{Q}_b|}\sum_{q\in\mathcal{Q}_b}\mathrm{BLEU1}_q
\end{equation}

where $\mathrm{BP}_q=1$ if $|\hat{y}_q|>|y_q^\star|$ and otherwise
$\mathrm{BP}_q=\exp(1-|y_q^\star|/|\hat{y}_q|)$. We use NLTK method-1
smoothing for sentence-level BLEU-1.

MemEye and Mem-Gallery use their native reference-based judge, whose parsed
score is $s_q\in\{0,0.5,1\}$; $\mathrm{Judge}_b=|\mathcal{Q}_b|^{-1}
\sum_{q\in\mathcal{Q}_b}s_q$. WorldMemArena-S labels each answer as correct, hallucination, or
omission and reports QA-C as its Judge:
\begin{equation}
    \mathrm{QA\text{-}C}=\mathrm{Judge}_{\mathrm{WMA\text{-}S}}
=\frac{|\{q\in\mathcal{Q}_{\mathrm{WMA\text{-}S}}:\ell_q=\mathrm{correct}\}|}{N_{\mathrm{valid}}}
\end{equation}

For WorldMemArena memory diagnostics, let $\mathcal{G}_s$ be the gold memory
points, $\mathcal{D}_s$ the stored add/update items, and $\mathcal{C}_s$ the
gold points judged covered by $\mathcal{D}_s$. With importance weight $w_g$,
Memory Recall is
\begin{equation}
    \texttt{Memory\_Recall}_s=
\frac{\sum_{g\in\mathcal{C}_s}w_g}{\sum_{g\in\mathcal{G}_s}w_g},
\qquad
\texttt{Memory\_Recall}=\operatorname{mean}_{s:\,|\mathcal{G}_s|>0}
\texttt{Memory\_Recall}_s.
\end{equation}

If $n_s^{\mathrm C}$ stored items are judged correct, Memory Correctness is
\begin{equation}
    \texttt{Memory\_Correctness}=\operatorname{mean}_{s:\,|\mathcal{D}_s|>0}
\frac{n_s^{\mathrm C}}{|\mathcal{D}_s|},
\end{equation}

implemented by the released evaluator as
$1-n_s^{\mathrm{incorrect}}/n_s^{\mathrm{checked}}$. For each update, the
post-session memory is classified as updated, both, or outdated, yielding
\begin{equation}
\texttt{Update\_Handling}=
\frac{N_{\mathrm{updated}}+0.5N_{\mathrm{both}}+0N_{\mathrm{outdated}}}{N_{\mathrm{total}}}.    
\end{equation}

Finally, if $\mathcal{V}_q$ is the gold evidence set and
$\mathrm{covered}_q$ is the number of evidence points found by the LLM judge in
the returned retrieval list, retrieval hit rate is
\begin{equation}
\texttt{retrieval\_hit\_rate}=\frac{1}{|\mathcal{Q}_{\mathrm{WMA\text{-}S}}|}
\sum_{q\in\mathcal{Q}_{\mathrm{WMA\text{-}S}}}
\frac{\mathrm{covered}_q}{|\mathcal{V}_q|},    
\end{equation}

where questions without gold evidence contribute zero. WMA metrics are first
computed per trajectory and then averaged across evaluated instances.

Table~\ref{tab:fine-grained-tasks} consolidates the fine-grained task
semantics used by Mem-Gallery and WorldMemArena-S. Shared capabilities are
listed once, while benchmark-specific labels and definitions are retained.

\begin{table*}[t]
\centering
\small
\setlength{\tabcolsep}{4pt}
\renewcommand{\arraystretch}{1.08}
\caption{Fine-grained task definitions used by Mem-Gallery and WorldMemArena-S. A dash indicates that the benchmark has no directly corresponding task label.}
\label{tab:fine-grained-tasks}
\begin{tabularx}{\linewidth}{@{}>{\raggedright\arraybackslash}p{0.18\linewidth}>{\raggedright\arraybackslash}p{0.37\linewidth}>{\raggedright\arraybackslash}X@{}}
\toprule
Capability & Mem-Gallery & WorldMemArena-S \\
\midrule
Fact recall & \textbf{FR} (Factual Retrieval): retrieve an explicit fact from the conversation. & \textbf{FR} (Fact Recall): retrieve one concrete fact stated earlier in the trajectory. \\
Multi-entity reasoning & \textbf{MR} (Multi-entity Reasoning): combine retrieved information about multiple entities to infer an answer. & -- \\
Temporal reasoning & \textbf{TR} (Temporal Reasoning): resolve a time-dependent question. & \textbf{TR} (Temporal Reasoning): reason about timing, ordering, or duration. \\
Visual or cross-modal reasoning & \textbf{VR} (Visual-centric Reasoning): use visual information in addition to the dialogue. & \textbf{CMR} (Cross-modal Reasoning): combine textual and visual memory. \\
Test-time learning & \textbf{TTL} (Test-time Learning): learn visual knowledge from historical images and apply it during QA. & \textbf{TTL} (Test-Time Learning): apply a rule or skill taught earlier in the trajectory. \\
Visual search & \textbf{VS} (Visual-centric Search): find the image(s) matching a query and return their image IDs. & \textbf{VS} (Visual Search): locate an object or attribute across visual memory. \\
Conflict handling & \textbf{CD} (Conflict Detection): detect a contradiction between the history and the question. & \textbf{MC} (Memory Conflict): resolve disagreement between two memory items using context. \\
Knowledge reasoning or resolution & \textbf{KR} (Knowledge Resolution): resolve updates or conflicts by prioritizing the most recent information. & \textbf{KR} (Knowledge Reasoning): combine stored facts with general world knowledge. \\
Answer boundary & \textbf{AR} (Answer Refusal): answer ``Not mentioned'' when the information is absent. & -- \\
Dynamic update & -- & \textbf{DU} (Dynamic Update): report the latest value after a queried fact is overwritten. \\
Visual fact recall & -- & \textbf{VFR} (Visual Fact Recall): recover a fact anchored to a specific image. \\
Visual update & -- & \textbf{VU} (Visual Update): report the most recent state after a visual attribute changes. \\
\bottomrule
\end{tabularx}
\end{table*}

\subsection{Implementation Details}
\label{app:experiment-implementation}

We use Qwen3-VL-8B, GPT-4.1-mini, and Qwen3.8-27B memory backbones. The selected backbone performs memory construction, agentic retrieval,and answer generation in MHG-Mem.
We instantiate the text encoder $\phi_T$ with
\texttt{all-MiniLM-L6-v2}, and the image and cross-modal encoders
$(\phi_I,\psi_T)$ with \texttt{google/siglip2-base-patch16-384}. Qwen3-VL-8B and Qwen3.8-27B are served using vLLM on four NVIDIA A800 GPUs, while GPT-4.1-mini is accessed through the official API. All models apply 65,536 max token limit and 0 temperature. Each result is averaged over three runs.

In memory construction stage, the context $\mathcal{C}_t$ contains up to three
preceding rounds from the same session, and at most eight images from
$\mathcal{I}_t$ are passed to the construction models. Occurrence linking
uses the top $K=3$ candidate hyperedges in Eq.~\eqref{eq:mhg-link-score},
with $w_T=w_I=0.5$. Evolution synthesis retains only records $z$ grounded in at least
two facts in $\mathcal{F}_z$ and observations from at least two sessions. For WorldMemArena-S,
$\widehat{\mathcal{Z}}_r$ is refreshed after every visible session checkpoint;
for the static MemEye and Mem-Gallery corpora, it is synthesized after the
available history $\mathcal{D}_{\leq T}$ has been constructed.

In memory retrieval stage, the controller $\pi_\theta$ follows the search loop in
Algorithm~\ref{alg:mhg-search} with $B_{\mathrm{step}}=16$ and $B_{\mathrm{img}}=10$. The entity-layer search returns at most 3 referent groups; the visual search returns at most 5 images. Evolution-layer search returns at most 3 relevant evolution aspects. The fact-layer search returns at most 30 candidates facts.
For MemEye and Mem-Gallery, source
inspection exposes the rounds referenced by $\rho(f)$; WorldMemArena-S exposes
the cited FactNodes $f\in\mathcal{F}$. We follow the benchmark-specific QA and
judging procedures defined above.

\subsection{Baselines}
\label{app:experiment-baselines}

\paragraph{A-Mem.}
A-Mem~\citep{xu2025amem} organizes textual memories as interconnected
notes inspired by the Zettelkasten method. It enriches each note with
contextual descriptions and indexing attributes, establishes links to
related notes, and updates existing memories as new observations arrive.

\paragraph{MemoryOS.}
MemoryOS~\citep{kang2025memoryos} maintains short-term, mid-term, and
long-term memory. Storage, updating, retrieval, and generation modules
coordinate the transfer and consolidation of conversational information
across these tiers to support persistent user knowledge and preferences.

\paragraph{SimpleMem.}
SimpleMem~\citep{liu2026simplemem} compresses interactions into compact,
self-contained textual memory units. Its pipeline combines semantic
structured compression, online semantic synthesis, and intent-aware
retrieval planning to reduce redundancy and assemble query-relevant
context. The unimodal category refers to textual memory representations.

\paragraph{HyperMem.}
HyperMem~\citep{yue2026hypermem} organizes long-term conversations into
topic, episode, and fact levels. Hyperedges group associated episodes
and facts, while hybrid lexical--semantic indexing and coarse-to-fine
retrieval recover related information across these levels.

\paragraph{M2A.}
M2A~\citep{feng2026m2a} combines an immutable raw-message store with a
semantic-memory store of higher-level observations. A conversational
agent and a memory manager coordinate memory updates and iterative
retrieval, retaining links between semantic memories and their original
multimodal evidence.

\paragraph{MMA.}
MMA~\citep{lu2026mma} is a confidence-aware multimodal memory agent.
It combines semantic relevance with memory reliability, estimated from
source credibility, temporal decay, and conflict-aware consensus, to
reduce reliance on stale or conflicting observations.

\paragraph{SRAG.}
SRAG denotes the multimodal Semantic RAG baseline in
MemEye~\citep{guo2026memeye}. It ranks historical records by combining
text and image embedding similarities and supplies retrieved text and
associated original images to the answer model. It provides a direct
retrieval comparison to methods that construct additional memory
abstractions or relational structures.

\paragraph{UniversalRAG.}
UniversalRAG~\citep{yeo2025universalrag} routes queries to corpora
organized by modality and granularity. Modality-aware routing selects
the appropriate evidence source, while retrieval at different
granularities supplies context for answer generation.

\section{Additional Experiment Results and Analysis}
\label{additional_exp_analysis}

\begin{figure}[h]
\centering
\includegraphics[width=1.0\linewidth]{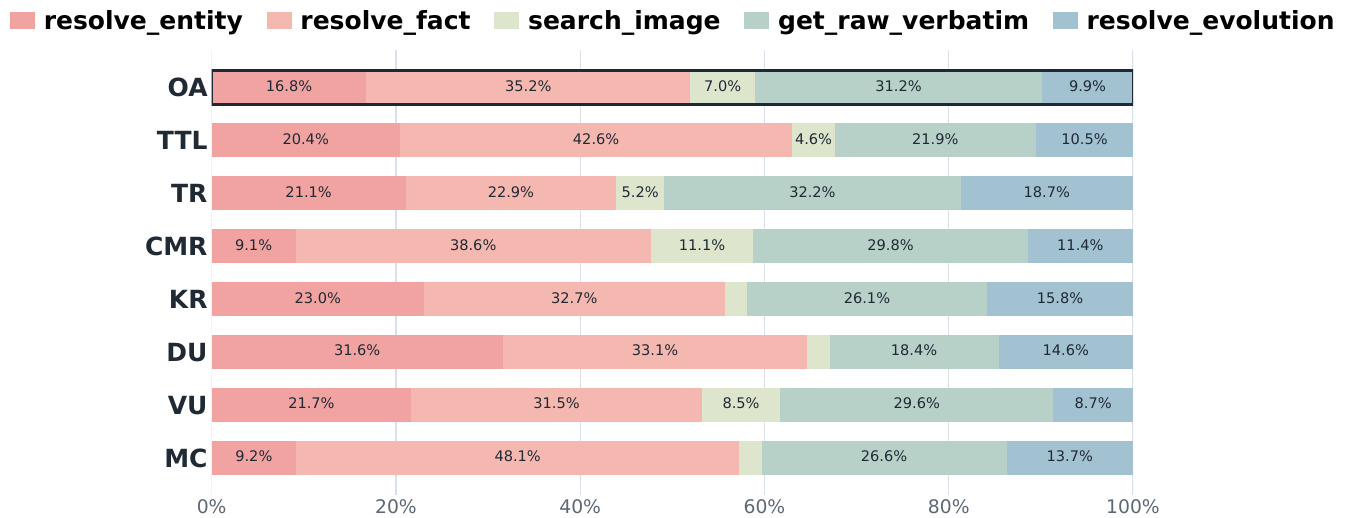}
\caption{Average tool utilization on WorldMemArena-S with Qwen3.8-27B, reported overall and by 7 task types: OA (Overall Average), TTL (Test-time Learning), TR (Temporal Reasoning), CMR (Cross-modal Reasoning), KR (Knowledge Reasoning), DU (Dynamic Updating), VU (Visual Update), and MC (Memory Conflict).}
\label{fig:tool_call_ratios}
\end{figure}

\subsection{Retrieval Tool Analysis}
Figure~\ref{fig:tool_call_ratios} reports average tool usage on WorldMemArena-S with Qwen3.8-27B. Hypergraph tools ($\texttt{resolve\_entity}$, $\texttt{resolve\_fact}$, and $\texttt{resolve\_evolution}$) and the inspection tool ($\texttt{get\_raw\_verbatim}$) dominate agentic search. In particular, $\texttt{resolve\_fact}$ and $\texttt{get\_raw\_verbatim}$ are invoked most frequently, suggesting that many questions can be traced through atomic facts and verified against original dialogue or visual evidence. Entity- and evolution-layer retrieval also accounts for substantial usage in TTL, TR, KR, DU, and VU tasks. This pattern suggests that, for temporal reasoning and multimodal updates, the agent links observations of the same entity and retrieves relevant state histories to reason about attribute updates, conflicts, and replacements.

\begin{wrapfigure}{r}{0.5\textwidth}
    \centering
    \includegraphics[width=\linewidth]{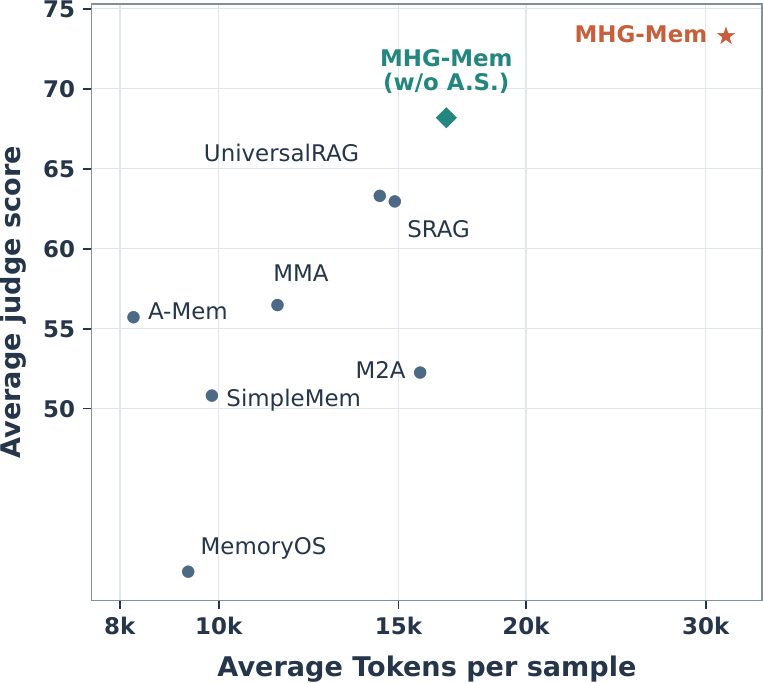}
    \caption{Trade-off between token consumption and QA performance among different baselines. w/o A.S. represents replace MHG-Mem's agentic search with top-k retrieval.}
    \label{fig:performance_tokens_tradeoff}
\end{wrapfigure}

\subsection{Cost Analysis}
Figure~\ref{fig:performance_tokens_tradeoff} compares token consumption and QA performance across methods. Following prior work, we measure the average total number of LLM input and output tokens per sample across memory construction, retrieval, and QA. For multimodal methods, this count also includes image input tokens determined by the corresponding LLM's visual encoding. We additionally report MHG-Mem (w/o A.S.), an ablation that replaces agentic search with top-$k$ hypergraph retrieval.

Among the baselines, unimodal memory methods that abstract images into textual captions consume fewer tokens but achieve lower QA performance, as fine-grained visual details are omitted. Among multimodal methods, RAG-based approaches (SRAG, MMA, and UniversalRAG) require fewer tokens than memory-based approaches (M2A and MHG-Mem) because they avoid explicit memory construction and iterative search. Although MHG-Mem improves benchmark-averaged QA performance by 8.20--9.98 points over the strongest baseline for each backbone, it consumes approximately 2x tokens, reflecting both the processing of multimodal information during memory construction and iterative tool calls during retrieval. The additional cost primarily arises from agentic search: MHG-Mem (w/o A.S.) requires only approximately 1--2k more tokens than multimodal baselines while still delivering performance gains of 5\%--15\%, offering a favorable balance between cost and performance. These results suggest using similarity-based hypergraph retrieval in resource-constrained settings and agentic hypergraph search when higher QA performance is the priority.

\begin{figure}[h]
\centering
\includegraphics[width=1.0\linewidth]{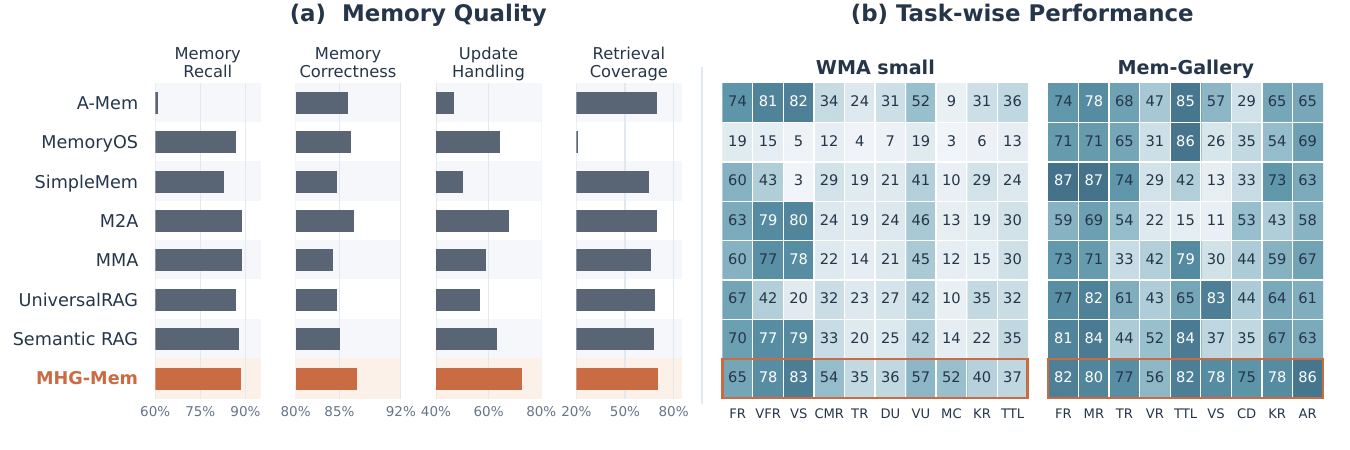}
\caption{(a) Memory quality and (b) fine-grained tasks performance on WorldMemArena-S
(WMA-S) and Mem-Gallery with Qwen3-VL-8B.}
\label{fig:qwen3vl8b_memory_and_tasks}
\end{figure}

\begin{figure}[h]
\centering
\includegraphics[width=1.0\linewidth]{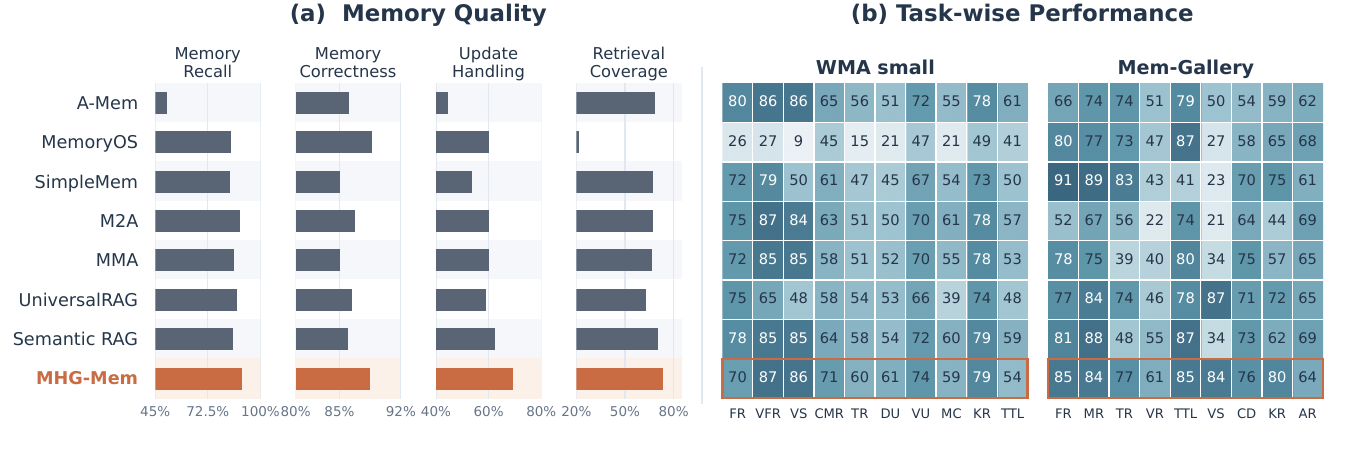}
\caption{(a) Memory quality and (b) fine-grained tasks performance on WorldMemArena-S
(WMA-S) and Mem-Gallery with GPT-4.1-mini.}
\label{fig:gpt41mini_memory_and_tasks}
\end{figure}

\subsection{Memory Quality and Fine-grained Analysis on More Backbones}
Figure~\ref{fig:qwen3vl8b_memory_and_tasks} and~\ref{fig:gpt41mini_memory_and_tasks} extend the memory quality and task-level analysis to Qwen3-VL-8B and GPT-4.1-mini, respectively. Across both backbones, MHG-Mem achieves the highest memory correctness and update handling, with retrieval coverage among the best. These results show that effective state maintenance coexists with accurate memory retention and strong evidence coverage, consistent with our objective of organizing evolving observations while keeping their supporting evidence accessible.

The most consistent task-level advantages concern visual evidence use and state evolution. On WorldMemArena-S, MHG-Mem outperforms all compared baselines in cross-modal reasoning, temporal reasoning, dynamic updates, and visual updates with both backbones. In particular, cross-modal reasoning improves over the strongest baseline by 20\% and 6\% with Qwen3-VL-8B and GPT-4.1-mini, respectively. On Mem-Gallery, visual-centric reasoning improves by 4\% and 6\%, while knowledge resolution improves by 5\% under both backbones; conflict detection also benefits in both settings. These recurring gains align with the two requirements motivating MHG-Mem: accessing relevant visual observations and integrating changes across interaction histories. They are consistent with the intended role of entity histories and state-specific evidence links in supporting comparison and verification across observations.

The improvements nevertheless depend on the task and backbone. For example, Mem-Gallery temporal reasoning improves by 3\% with Qwen3-VL-8B but remains 6\% below the strongest baseline with GPT-4.1-mini. MHG-Mem also trails the strongest factual-recall baseline on both benchmarks under both backbones. Thus, the memory-quality improvements do not translate uniformly into every downstream capability: visual evidence use and state updates show reproducible benefits, while factual recall remain areas for improvement.

\begin{table*}[!h]
    \centering
    \caption{Hyperparameter sensitivity at the construction and retrieval stages with Qwen3.8-27B. The construction-stage sweep varies the number of top-$K$ occurrence-linking candidates. The retrieval-stage sweep evaluates the agentic search budget $(B_{\mathrm{step}}, B_{\mathrm{img}})$ and reports the force-submit rate. LLM-Judge scores are reported. Avg. means the average judge score across three benchmark. Bold values are the best within each stage.}
    \label{tab:parameter-analysis}
    \setlength{\tabcolsep}{4.5pt}
    \renewcommand{\arraystretch}{1.08}
    \begin{adjustbox}{max width=\linewidth}
    \begin{tabular}{@{}llrrrrc@{}}
        \toprule
         & \textbf{Setting} & \textbf{MemEye} & \textbf{Mem-Gallery} & \textbf{WMA-S} & \textbf{Avg.} & \textbf{Force-submit (\%)} \\
        \midrule
        \rowcolor{gray!12}
        \multicolumn{7}{@{}l@{}}{\textbf{Construction stage}} \\
        & $K=1$ & 64.70 & 85.84 & 64.41 & 71.65 & --- \\
        & $K=2$ & 65.54 & 86.99 & 65.67 & 72.73 & --- \\
        & $K=3$ & 65.93 & 87.85 & 66.05 & 73.28 & --- \\
        & $K=4$ & 65.22 & 87.46 & \textbf{66.35} & 73.01 & --- \\
        & $K=5$ & \textbf{66.18} & \textbf{88.10} & 65.71 & \textbf{73.33} & --- \\
        \midrule
        \rowcolor{gray!12}
        \multicolumn{7}{@{}l@{}}{\textbf{Retrieval stage}} \\
        & $(B_{\mathrm{step}}, B_{\mathrm{img}})=(4, 4)$ & 59.44 & 81.79 & 60.63 & 67.29 & 37.3 \\
        & $(B_{\mathrm{step}}, B_{\mathrm{img}})=(10, 4)$ & 62.02 & 84.76 & 62.37 & 69.72 & 26.8 \\
        & $(B_{\mathrm{step}}, B_{\mathrm{img}})=(10, 8)$ & 63.91 & 86.30 & 63.10 & 71.10 & 24.6 \\
        & $(B_{\mathrm{step}}, B_{\mathrm{img}})=(16, 8)$ & 65.42 & 87.04 & 64.99 & 72.48 & 17.9 \\
        & $(B_{\mathrm{step}}, B_{\mathrm{img}})=(16, 10)$ & 65.73 & \textbf{87.85} & 66.05 & 73.21 & 17.5 \\
        & $(B_{\mathrm{step}}, B_{\mathrm{img}})=(24, 10)$ & 65.47 & 87.63 & 66.51 & 73.20 & \textbf{13.4} \\
        & $(B_{\mathrm{step}}, B_{\mathrm{img}})=(24, 16)$ & \textbf{66.09} & 87.46 & \textbf{66.80} & \textbf{73.45} & 13.5 \\
        \bottomrule
    \end{tabular}
    \end{adjustbox}
\end{table*}

\subsection{Hyperparameter Analysis}
\textbf{Analysis of the occurrence-linking candidate count $K$.}
The upper half of Table~\ref{tab:parameter-analysis} reports QA performance with different top-$K$ candidate counts for occurrence linking during memory construction. Across the tested values of $K$, scores vary by only 1.48, 2.26, and 1.94 points on MemEye, Mem-Gallery, and WorldMemArena-S, respectively. The average score is 73.28 at $K=3$, compared with 73.33 at the best tested value, $K=5$, indicating that performance is relatively insensitive to $K$. Since larger candidate sets increase context length with limited performance gains, we retrieve at most $K=3$ candidates for occurrence linking.

\textbf{Analysis of agentic budgets $B_{\mathrm{step}}$ and $B_{\mathrm{img}}$.}
The lower half of Table~\ref{tab:parameter-analysis} examines the effects of retrieval budgets. Increasing $B_{\mathrm{step}}$ or $B_{\mathrm{img}}$ allows longer retrieval trajectories or more image evidence, respectively, but also increases input token consumption. Raising the budgets from $(4,4)$ to $(16,8)$ improves the average QA score from 67.29 to 72.48, a gain of 5.19 percentage points, while reducing the forced-submission rate from 37.3\% to 17.9\%. These results suggest that restrictive budgets can prevent the agent from gathering sufficient evidence through hypergraph search. For example, questions requiring fine-grained visual distinctions among multiple entities may require inspecting more image crops than the budgets permit, leading to premature submission and lower performance. However, increasing $(B_{\mathrm{step}}, B_{\mathrm{img}})$ further to $(24,16)$ yields only a 0.24-point improvement over $(16,10)$ while consuming more tokens. To balance the cost and performance trade-off, we therefore set $(B_{\mathrm{step}}, B_{\mathrm{img}})\leq(16,10)$ for experiment.

\subsection{More Case Studies}
Figures~\ref{fig:case2} and~\ref{fig:case3} illustrate how MHG-Mem reconstructs entity histories through adaptive retrieval and source verification. In Figure~9, the agent retrieves the faucet's evolution record and traces its supporting facts to original frames, recovering the complete \emph{off--running--off} sequence; the baselines omit states or fail to answer. In Figure~10, retrieval begins with entity resolution, followed by evolution retrieval and source inspection. This recovers the oak-log count sequence $1 \rightarrow 2 \rightarrow 5$ and verifies the latest count from the inventory image, whereas UniversalRAG and M2A recover only $1 \rightarrow 2$. Both cases highlight the value of linking observations under a shared entity identity while preserving state-specific evidence for cross-session synthesis and visual verification.

\begin{figure}[h]
\centering
\includegraphics[width=1.0\linewidth]{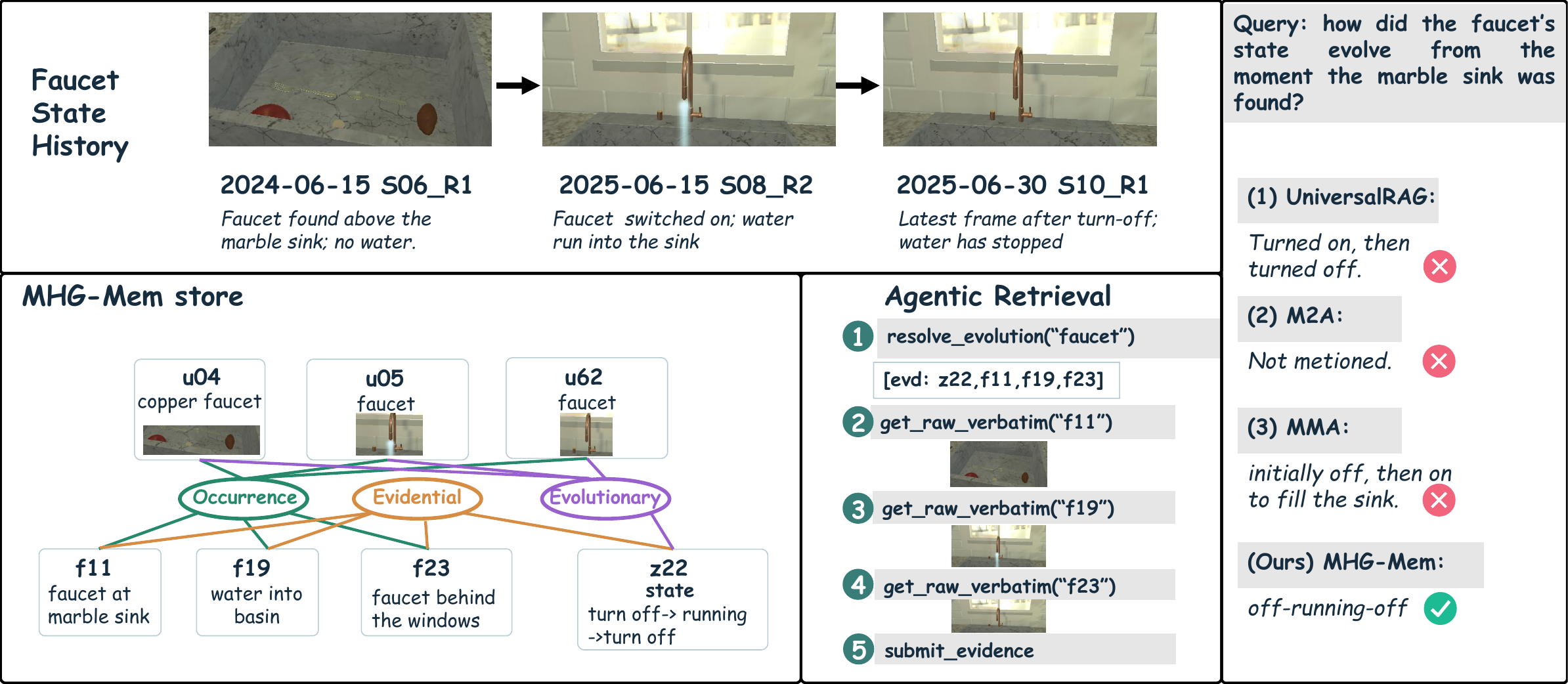}
\caption{Case study in an embodied agent scenario. MHG-Mem records the faucet's state throughout the multimodal agent's action history.}
\label{fig:case2}
\end{figure}

\begin{figure}[h]
\centering
\includegraphics[width=1.0\linewidth]{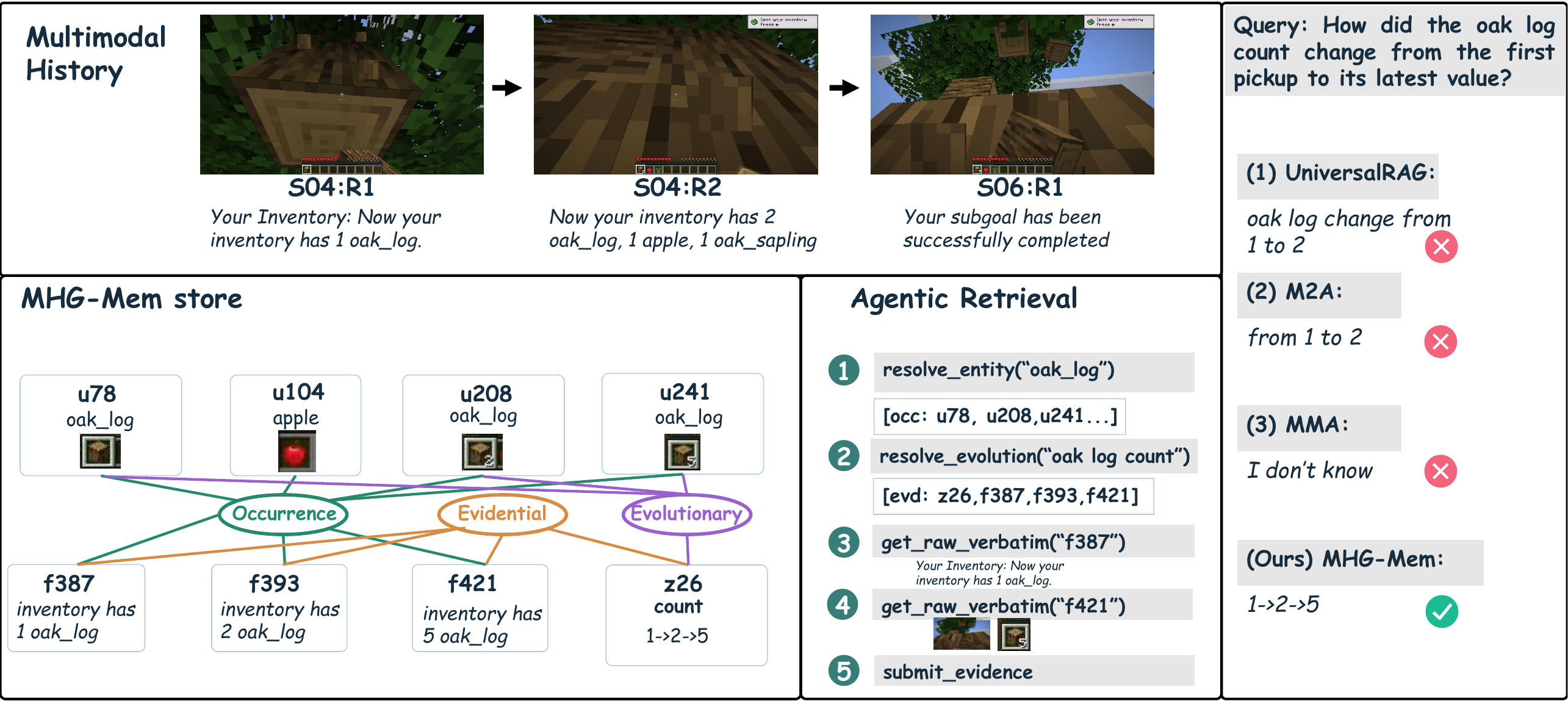}
\caption{Case study in a GUI agent scenario. MHG-Mem records the count state of oak log in Minecraft throughout the multimodal agent's action history.}
\label{fig:case3}
\end{figure}

\clearpage
\section{Prompt Templates}
\label{app:prompts}
The five prompt templates below cover memory construction (Appendix~\ref{app:mhg-construction-details})
and agentic search (Table~\ref{tab:mhg-tools}). We preserve prompt wording and
placeholders, omit Python assignment wrappers, and adjust line wrapping for readability. For additional LLM-as-a-judge and QA prompts, we use the original benchmark-provided prompts to ensure fair comparisons.

\subsection{Atomic fact construction}
\label{app:prompt-fact}
Converts the current round into standalone facts while retaining temporal and visual details.

\begin{mhgprompt}{promptsage}{01\quad Atomic fact construction}
\par
You convert ONE conversation round into a set of ATOMIC MEMORIES \textemdash{} short,\allowbreak{} standalone facts that a separate answerer will later retrieve to answer questions. You are given the round's dialogue text,\allowbreak{} its image caption(s),\allowbreak{} and its REAL image(s). You may also see up to three EARLIER rounds of the same conversation as read-only context (their text and image captions only \textemdash{} no earlier pictures),\allowbreak{} and a list of EXISTING MEMORIES already stored from past rounds. Write memories ONLY for the CURRENT round; use the earlier rounds and the current image(s) just to understand the round and to resolve references.\par
\par
Read BOTH the caption AND the real pixels of the current round. The caption often spells a detail out in words (card names,\allowbreak{} on-screen text,\allowbreak{} dashboard labels,\allowbreak{} message transcripts); the pixels carry details the caption omits (colors,\allowbreak{} counts,\allowbreak{} spatial layout,\allowbreak{} what is actually shown). Capture from both.\par
\par
\textbf{\#\# ADD-ONLY against EXISTING MEMORIES}\par
If an \textasciigrave{}\#\# EXISTING MEMORIES\textasciigrave{} list is given,\allowbreak{} treat it as memories that are ALREADY STORED. Your job is ADDITIVE: output the facts from the CURRENT round that are NOT already represented there. Do not restate,\allowbreak{} paraphrase,\allowbreak{} or lightly reword an existing memory. Never create a memory from the existing list alone \textemdash{} every memory you output must be supported by the CURRENT round.\par
- Default to WRITING a fact. Suppress a candidate ONLY when it is the SAME fact about the SAME moment as an existing memory. When in doubt,\allowbreak{} WRITE it \textemdash{} a missed fact can never be recovered,\allowbreak{} a near-duplicate is cheap.\par
- The mere presence of the same entity,\allowbreak{} brand,\allowbreak{} game,\allowbreak{} route,\allowbreak{} dashboard,\allowbreak{} or template in an existing memory does NOT make the current round a duplicate. A new event,\allowbreak{} a new state,\allowbreak{} a changed value,\allowbreak{} a different date/\allowbreak{}session,\allowbreak{} a new opinion or conclusion the user draws,\allowbreak{} or any new detail is ALWAYS a new memory \textemdash{} write it.\par
\par
Write each memory to this STANDARD:\par
- ATOMIC \& self-contained: one fact per memory \textemdash{} a single event,\allowbreak{} state,\allowbreak{} attribute,\allowbreak{} or relationship \textemdash{} understandable on its own. Begin from the concrete named subject.\par
- ABSOLUTE time: anchor every time reference to the Observation Date given. Convert "today /\allowbreak{} tomorrow /\allowbreak{} yesterday /\allowbreak{} last week /\allowbreak{} this Monday" to an explicit calendar date (YYYY-MM-DD). Never store a bare relative time or weekday.\par
- DESCRIBE the state,\allowbreak{} do NOT compute the answer. Record what is shown or said \textemdash{} list the cards in a hand,\allowbreak{} the panels on a dashboard and their values,\allowbreak{} the people in a chat and who said what. Do NOT add totals,\allowbreak{} counts across rounds,\allowbreak{} winners,\allowbreak{} comparisons,\allowbreak{} or any conclusion the data merely implies. Faithful state,\allowbreak{} never a frozen verdict.\par
- CAPTURE what a later question could ask: numbers,\allowbreak{} colors,\allowbreak{} spatial relations (left /\allowbreak{} right /\allowbreak{} above /\allowbreak{} below /\allowbreak{} lower-left /\allowbreak{} behind /\allowbreak{} beside /\allowbreak{} under),\allowbreak{} in-image text,\allowbreak{} per-item attributes,\allowbreak{} who did what,\allowbreak{} and any change from the earlier context.\par
- Only return an empty list when the CURRENT round genuinely contains NOTHING beyond what is already stored (e.g. a bare "ok",\allowbreak{} "thanks",\allowbreak{} "correct" with no new content of its own).\par
\par
Return ONLY JSON:\par
\{"atomic\_facts":["\textless{}one standalone memory\textgreater{}",\allowbreak{}"\textless{}another standalone memory\textgreater{}"]\}\par
If the round carries nothing new worth remembering,\allowbreak{} return \{"atomic\_facts":[]\}.\par
\end{mhgprompt}

\clearpage

\subsection{Multimodal entity extraction}
\label{app:prompt-entity}
Extracts nameable entities and grounds visual or multimodal mentions in image regions.

\begin{mhgprompt}{promptblue}{02\quad Multimodal entity extraction}
\par
You extract MULTIMODAL ENTITIES from ONE conversation round (its dialogue text AND its image,\allowbreak{} if any).\par
\par
An ENTITY is a persistent,\allowbreak{} nameable THING that could be asked about later \textemdash{} for example: a brand,\allowbreak{} logo,\allowbreak{} product. This list is illustrative,\allowbreak{} not exhaustive \textemdash{} extract any concrete,\allowbreak{} nameable thing a later question could point at\par
\par
Give each entity ONE modality,\allowbreak{} decided by WHERE you learned about it:\par
- \textasciigrave{}visual\textasciigrave{}: something you can SEE in the image. This is the DEFAULT for anything in the picture \textemdash{} objects,\allowbreak{} people,\allowbreak{} places,\allowbreak{} logos,\allowbreak{} products,\allowbreak{} signs,\allowbreak{} on-screen panels. Give it a short specific name (e.g. "red delivery truck",\allowbreak{} "Coca-Cola logo"). bbox required.\par
- \textasciigrave{}multimodal\textasciigrave{}: visible in the image AND also named in the dialogue text. Name it from the dialogue; the box is on the region that shows it. bbox required.\par
- \textasciigrave{}text\textasciigrave{}: named in the dialogue text (what the USER or ASSISTANT actually typed) but NOT visible in the image. bbox is null. This is RARE \textemdash{} use it ONLY for things mentioned in the conversation that you cannot see.\par
\par
For every \textasciigrave{}visual\textasciigrave{} and \textasciigrave{}multimodal\textasciigrave{} entity,\allowbreak{} bbox is absolute pixels \textasciigrave{}[x0,\allowbreak{}y0,\allowbreak{}x1,\allowbreak{}y1]\textasciigrave{} in the image's pixel frame,\allowbreak{} tight around the thing.\par
\par
Return ONLY JSON:\par
\{"entities":[\{"name":"...",\allowbreak{}"type":"\textless{}brand|\allowbreak{}person|\allowbreak{}character|\allowbreak{}object|\allowbreak{}product|\allowbreak{}place|\allowbreak{}animal|\allowbreak{}organization|\allowbreak{}other\textgreater{}",\allowbreak{}"modality":"visual|\allowbreak{}multimodal|\allowbreak{}text",\allowbreak{}"bbox":[x0,\allowbreak{}y0,\allowbreak{}x1,\allowbreak{}y1] or null\}]\}\par
\par
Rules:\par
- Only extract things written in the dialogue or clearly visible in the image \textemdash{} never invent.\par
- If nothing qualifies,\allowbreak{} return \{"entities":[]\}.\par
\end{mhgprompt}

\subsection{Occurrence linking}
\label{app:prompt-occurrence}
Associates a candidate mention with a stored referent, creates a new referent, or rejects an invalid extraction.

\begin{mhgprompt}{promptsage}{03\quad Occurrence linking}
\par
You decide how to store ONE candidate entity into a memory of already-known entities.\par
\par
You are given:\par
- ROUND CONTEXT: the conversation round the candidate was found in (its text,\allowbreak{} and the full round image if any) \textemdash{} use it to understand what the candidate actually is.\par
- CANDIDATE: its name,\allowbreak{} type,\allowbreak{} modality,\allowbreak{} and,\allowbreak{} if visual,\allowbreak{} its cropped region (shown below).\par
- MATCHES: most similar to the candidate (by text and visual similarity),\allowbreak{} each with its best-matching stored region (shown below).\par
\par
Choose EXACTLY ONE action:\par
- \textasciigrave{}update\textasciigrave{}: the candidate is the SAME real-world thing as one MATCH\textemdash{} link to that entity id. Also use \textasciigrave{}update\textasciigrave{} for a repeat sighting of a thing already stored,\allowbreak{} to record this new occurrence.\par
- \textasciigrave{}add\textasciigrave{}: the candidate is a REAL entity but none of the MATCHES is the same real-world thing.\par
- \textasciigrave{}skip\textasciigrave{}: ONLY when the candidate is NOT a real,\allowbreak{} storable entity (an extraction mistake or abstract concept).\par
\par
Return ONLY JSON:\par
\textasciigrave{}\{\{"action":"update|\allowbreak{}add|\allowbreak{}skip",\allowbreak{}"target\_entity\_id":"\textless{}id,\allowbreak{} only if update\textgreater{}",\allowbreak{}"add\_alias":"\textless{}optional new alias\textgreater{}",\allowbreak{}"reason":"\textless{}one line\textgreater{}"\}\}\textasciigrave{}\par
\end{mhgprompt}

\clearpage

\subsection{State evolution synthesis}
\label{app:prompt-evolution}
Identifies changing aspects and binds each reported state to its supporting fact identifier.

\begin{mhgprompt}{promptlavender}{04\quad State evolution synthesis}
\par
You are given ONE referent (a thing that recurs in a long conversation) and the time-ordered list of ATOMIC FACTS recorded about it. Your job is to report which ASPECTS of this referent actually CHANGED over time.\par
\par
An ASPECT is a single named property that can hold different values at different times \textemdash{} for example: hand\_size,\allowbreak{} location,\allowbreak{} owner,\allowbreak{} color,\allowbreak{} price,\allowbreak{} status,\allowbreak{} score,\allowbreak{} version,\allowbreak{} participants,\allowbreak{} weight.\par
\par
\textbf{\#\# What counts as a change}\par
A change means: the SAME aspect of the SAME referent holds DIFFERENT values at different times. For example:\par
- "is in the kitchen" then "is in the bedroom"-\textgreater{} aspect \textasciigrave{}location\textasciigrave{} CHANGED. Report it.\par
- "costs \$12" then "costs \$9"-\textgreater{} aspect \textasciigrave{}price\textasciigrave{} CHANGED. Report it.\par
\par
\textbf{\#\# What does NOT count as a change (this is the common case \textemdash{} read carefully)}\par
- A value stated repeatedly,\allowbreak{} or restated in different words,\allowbreak{} is NOT a change.\par
- New,\allowbreak{} additional details about DIFFERENT aspects are NOT a change unless one aspect's own value is replaced by another value.\par
\par
\textbf{\#\# How to write what you find}\par
For each aspect that genuinely changed:\par
- \textasciigrave{}aspect\textasciigrave{}: a short lowercase snake\_case key naming the property (\textasciigrave{}hand\_size\textasciigrave{},\allowbreak{} not "the size of the hand").\par
- \textasciigrave{}states\textasciigrave{}: two or more entries in TIME ORDER. Each entry MUST carry:\par
- \textasciigrave{}value\textasciigrave{}: the value at that time,\allowbreak{} as short as possible (a number,\allowbreak{} a color,\allowbreak{} a place,\allowbreak{} a name).\par
- \textasciigrave{}fact\_id\textasciigrave{}: the id of the fact you read that value from. It MUST be one of the fact ids listed below,\allowbreak{} copied exactly. Never invent an id.\par
- \textasciigrave{}summary\textasciigrave{}: one sentence DESCRIBING the observed states in order,\allowbreak{} with their dates.\par
- \textasciigrave{}update\_type\textasciigrave{}: one of\par
- \textasciigrave{}replacement\textasciigrave{} \textemdash{} the new value replaces the old one (a measurement,\allowbreak{} a location,\allowbreak{} a price)\par
- \textasciigrave{}correction\textasciigrave{} \textemdash{} the later fact says the earlier one was wrong\par
- \textasciigrave{}addition\textasciigrave{}  \textemdash{} the value grew by accumulating items rather than being replaced\par
\par
Return ONLY JSON:\par
\{"aspects":[\{"aspect":"\textless{}snake\_case\textgreater{}",\allowbreak{}"states":[\{"value":"\textless{}v1\textgreater{}",\allowbreak{}"fact\_id":"\textless{}id\textgreater{}"\},\allowbreak{}\{"value":"\textless{}v2\textgreater{}",\allowbreak{}"fact\_id":"\textless{}id\textgreater{}"\}],\allowbreak{}"summary":"\textless{}one sentence describing the states in order\textgreater{}",\allowbreak{}"update\_type":"replacement|\allowbreak{}correction|\allowbreak{}addition"\}]\}\par
If no aspect of this referent changed,\allowbreak{} return exactly \{"aspects":[]\}.\par
\end{mhgprompt}

\clearpage

\subsection{Hypergraph-based agentic search}
\label{app:prompt-retrieval}
Guides sequential tool use to assemble evidence for the separate answer model.

The placeholders \texttt{[MEMORY\_SCHEMA]}, \texttt{[TOOL\_SCHEMA]}, and \texttt{\{max\_steps\}} denote the memory schema, tool definitions, and search budget, respectively. The five gathering tools collect evidence; \texttt{submit\_evidence} terminates retrieval.

\begin{mhgprompt}{promptblue}{05\quad Hypergraph-based agentic search}
\par
You are a hypergraph memory retriever. Given a QUESTION,\allowbreak{} you SEARCH a long-term multimodal memory of past conversation rounds and assemble an EVIDENCE CHAIN for a separate QA model. You do NOT answer the question yourself.\par
\par
\textbf{\#\# What is in memory: [MEMORY\_SCHEMA]}\par
\par
\textbf{\#\# Tool Definition: [TOOL\_SCHEMA]}\par
\par
\textbf{\#\# Format \textemdash{} EXACTLY ONE action per reply,\allowbreak{} nothing after it}\par
Each reply MUST be exactly:\par
\textasciigrave{}\textasciigrave{}\textasciigrave{}\par
Action: \textless{}one tool name\textgreater{}\par
Action Input: \{\{\textless{}json\textgreater{}\}\}\par
\textasciigrave{}\textasciigrave{}\textasciigrave{}\par
STRICT RULES:\par
- Output ONE Action,\allowbreak{} ONE Action Input,\allowbreak{} then STOP. Do NOT write a second Thought/\allowbreak{}Action in the same reply \textemdash{} you will only see the result of the FIRST action,\allowbreak{} so any extra ones are wasted.\par
- After \textasciigrave{}Action Input: \{\{...\}\}\textasciigrave{} write NOTHING more. Do not narrate,\allowbreak{} do not plan further steps.\par
- you will receive \textasciigrave{}Observation: \{\{...\}\}\textasciigrave{} (and any images). Read it,\allowbreak{} then send your NEXT single action.\par
When the evidence is sufficient,\allowbreak{} call submit\_evidence.\par
\par
\textbf{\#\# Strategy}\par
- You have five hypergraph gathering tools; use whichever fit the question. Call ONE per turn and read its Observation before the next.\par
- If the question hinges on a visual/\allowbreak{}pixel detail,\allowbreak{} open the most promising facts' rounds with get\_raw\_verbatim (or use search\_image) so you actually SEE the images before citing.\par
- Be efficient: you have at most \{max\_steps\} steps. Submit as soon as the evidence is sufficient.\par
\end{mhgprompt}

\end{document}